\documentclass{article} %
\usepackage{iclr2026_conference,times}

\usepackage{amsmath,amsfonts,bm}

\def\eqref#1{equation~\ref{#1}}

\def\1{\bm{1}}

\DeclareMathAlphabet{\mathsfit}{\encodingdefault}{\sfdefault}{m}{sl}
\SetMathAlphabet{\mathsfit}{bold}{\encodingdefault}{\sfdefault}{bx}{n}

\usepackage{hyperref}
\usepackage{url}
\usepackage{graphicx}
\usepackage{amsmath}
\usepackage{amssymb}
\usepackage{amsthm}
\usepackage{booktabs}
\usepackage{arydshln}
\usepackage{multirow}
\usepackage{array}
\usepackage{caption}
\usepackage{subcaption}
\usepackage{wrapfig}

\newtheorem{proposition}{Proposition}

\title{Multi-Resolution Flow Matching: Training-Free Diffusion Acceleration via Staged Sampling}

\author{%
Xingyu Zheng$^{1,2}$ \quad
Xianglong Liu$^{1,2}$\thanks{Corresponding author: Xianglong Liu \textless\url{xlliu@buaa.edu.cn}\textgreater.} \quad
Yifu Ding$^{1,2,3}$ \quad
Weilun Feng$^{4,5}$ \quad
Junqing Lin$^{6}$ \\
{\bf Jinyang Guo$^{1,7}$ \quad
Haotong Qin$^{8}$} \\
\normalfont
$^{1}$State Key Laboratory of Complex \& Critical Software Environment, Beihang University \\
$^{2}$School of Computer Science and Engineering, Beihang University \\
$^{3}$Nanyang Technological University \\
$^{4}$State Key Laboratory of AI Safety, Institute of Computing Technology, Chinese Academy of Sciences \\
$^{5}$University of Chinese Academy of Sciences \quad
$^{6}$University of Science and Technology of China \\
$^{7}$School of Artificial Intelligence, Beihang University \quad
$^{8}$ETH Z\"urich
}

\begin{document}

\maketitle

\begin{figure}[h]
\begin{center}
\includegraphics[width=\linewidth]{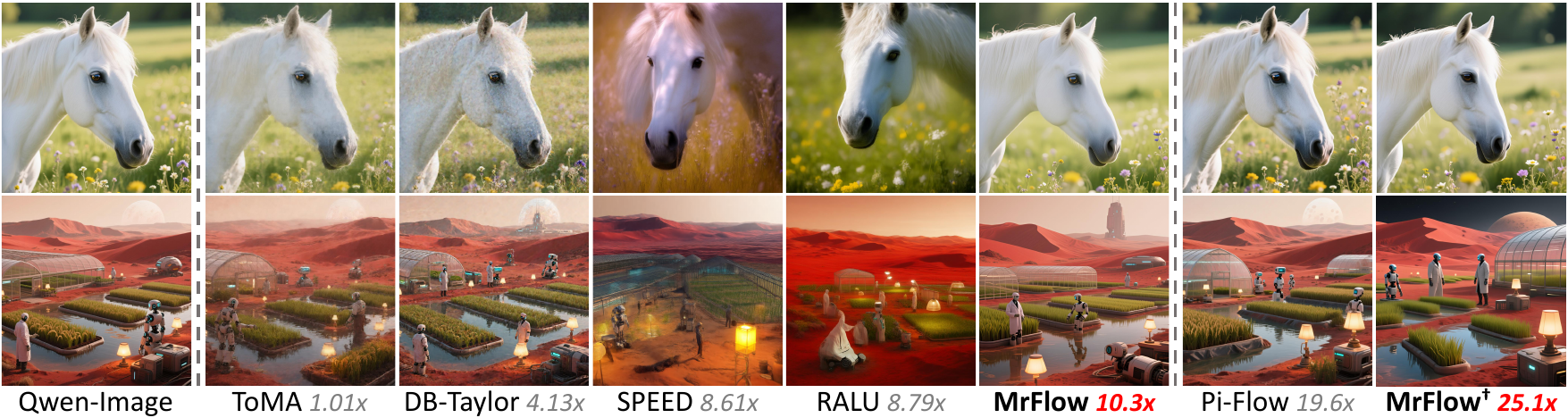}
\end{center}
\caption{Qualitative comparison between MrFlow and various methods on Qwen-Image. The dashed lines separate the pretrained model, training-free strategies, and strategies that rely on or exploit timestep distillation.}
\label{fig:showcase-compare}
\end{figure}

\begin{abstract}
  Hardware-agnostic strategies for accelerating text-to-image diffusion, such as timestep distillation and feature caching, can reduce inference time without custom kernels or system-level optimization. Among them, multi-resolution generation strategies have recently received broad attention, attaining more than $5\times$ speedup without any training. However, the design of performing upsampling in the latent space, together with the selective modification of partial regions, causes these methods to exhibit noticeable blurring or artifacts. To this end, we propose \textbf{MrFlow}, a training-free multi-resolution acceleration strategy for pretrained flow-matching models built upon a staged low-to-high-resolution pipeline. MrFlow first rapidly generates the main structure at low resolution, then performs super-resolution in the pixel space using a lightweight pretrained GAN-based model, subsequently injects low-strength noise to enable high-frequency resampling, and finally refines the details at high resolution. Quantitative and qualitative results on FLUX.1-dev and Qwen-Image show that MrFlow exploits the quadratic token reduction and reduced step requirement of low-resolution sampling to achieve $10\times$ end-to-end acceleration while keeping OneIG within a $1\%$ gap relative to that before acceleration, significantly surpassing other training-free acceleration strategies, and requiring no training or runtime dynamic identification whatsoever. MrFlow can further be directly combined orthogonally with pre-trained timestep distillation strategies, achieving even higher generation acceleration of up to $25\times$.
\end{abstract}

\section{Introduction}

Diffusion built upon Transformers as the main architecture~\citep{Peebles2023DiT} and flow matching~\citep{Lipman2023FlowMatching, Liu2023RectifiedFlow} as the paradigm has become the mainstream in image generation tasks. However, accompanying the scaling of generation quality is the scaling of the model's computational cost. Taking Qwen-Image-20B~\citep{QwenImage2025} as an example, it requires up to $47$s on diffusion sampling when performing a $1024\times1024$ text-to-image task on an Nvidia A100. As a result, numerous works seek to accelerate the sampling-based generation of diffusion. Apart from strategies that rely on hardware or systems to implement acceleration, such as quantization~\citep{Li2023QDiffusion} and efficient attention~\citep{Dao2022FlashAttention}, researchers have found that on diffusion one can reduce the number of computed tokens to achieve direct acceleration. Timestep distillation~\citep{Salimans2022ProgressiveDistillation, Luo2023LCM} is one of the most effective acceleration strategies, capable of reducing the Number of Function Evaluations (NFEs) from $50$-$100$ down to $1$-$4$. Feature caching~\citep{TeaCache2024} aims to exploit the similarity between adjacent timesteps, reusing historical features to skip the computation of partial layers or structures. However, these methods either depend heavily on training, or can only attain limited speedup of less than $4\times$. Some works that compose multiple resolution levels~\citep{LSSGen, RALU} exploit the spatial characteristics of the image modality and successfully achieve more than $5\times$ speedup in a training-free manner, but most still rely on runtime dynamic identification of image regions, and the operation of performing upsampling in the latent space leads to unavoidable artifacts or blurring.

In this paper, we design \textbf{MrFlow}, a multi-resolution sampling strategy oriented to modern diffusion. After generating the main image structure starting from random noise at low resolution (LR), it enlarges the image in the original pixel space via a super-resolution (SR) network, then injects low-strength noise into the VAE-reencoded latent to enable high-frequency resampling, and finally performs fast sampling in the high-resolution (HR) latent space to refine image details. MrFlow exploits the structure preservation and detail discrepancy of spatial information that images naturally possess at different resolutions, achieving efficient coarse-to-fine generation. This method: (1) fully exploits the fast structure-generation capability of the pretrained diffusion at low resolution, where not only the per-step execution time can be accelerated by about $4\times$, but also fewer timesteps are required; (2) adopts a pretrained lightweight GAN-based super-resolution network in the pixel space, preserving the low-resolution structural information while adding high-frequency signals; (3) the low-strength noise weakens potentially erroneous SR-imposed high-frequency details, allowing the high-resolution flow prior to resample and correct them; (4) refines at high resolution the image details that super-resolution can hardly fully realize, and the naturally straighter sampling trajectory near the clean-image end enables the resampling to complete rapidly as well. The overall framework of MrFlow is illustrated in Figure~\ref{fig:framework}.

\begin{figure}[t]
\begin{center}
\includegraphics[width=\linewidth]{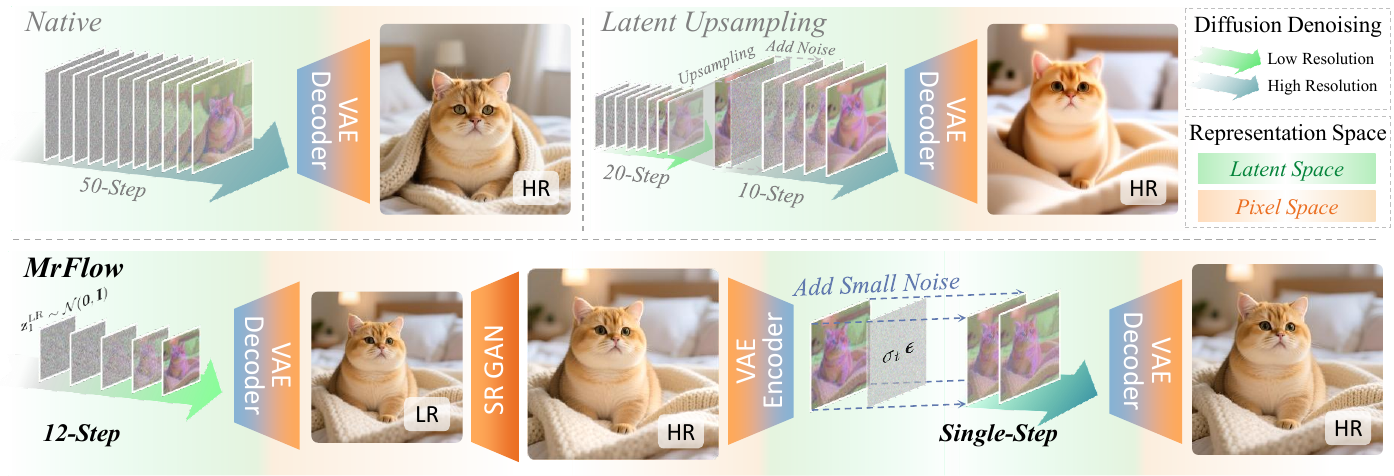}
\end{center}
\caption{The framework of MrFlow. The compared strategies include the native inference scheme and methods that perform upsampling in the latent space such as LSSGen, RALU and SPEED.}
\label{fig:framework}
\end{figure}

Extensive experiments on evaluation metrics such as OneIG-Bench and on real sample generation show that MrFlow can not only achieve generation results within a $1\%$ gap relative to the native trajectory-flow generation at a speedup of more than $10\times$, but also require no training or runtime dynamic statistics and identification, significantly surpassing various other training-free acceleration strategies in both generation quality and speedup. MrFlow can further be directly combined orthogonally with pre-trained timestep distillation strategies, achieving even higher generation acceleration of up to more than $25\times$.

In summary, our contributions are as follows:

\begin{itemize}
\item A multi-resolution acceleration pipeline: the overall procedure comprises low-resolution latent-space sampling, VAE decoding, pixel-space super-resolution, VAE encoding, noise injection, high-resolution latent-space sampling, and VAE decoding. This paradigm aims to harness the efficient image-structure generation capability of low resolution and the detail-semantics processing capability unique to high resolution, and makes clear that pixel-space super-resolution preserves structure, while latent-space noise injection enables high-frequency resampling before final refinement.
\item Principle innovation and analysis of each stage: we not only clarify the multi-resolution pipeline from an acceleration perspective, but also analyze in detail the advantages of the designed scheme in the acceleration of each stage. At the low-resolution stage, we point out that not only the per-step generation efficiency naturally enjoys a quadratic-level speedup, but also its efficient semantic generation capability makes the timesteps naturally easier to set fewer; at the super-resolution stage, we explicitly point out that one should employ a pretrained GAN-based super-resolution network in the pixel space to maximally preserve semantic information; at the noise-injection stage, we point out that the purpose of the noise injection is to reduce the influence of potentially erroneous SR high-frequency details, so that the high-resolution flow prior can resample them with only low strength; at the high-resolution stage, we point out that one can exploit the property that the flow trajectory is straighter near the image side, completing the sampling with very few timesteps.
\item Combining operational flexibility and simplicity, high speedup, good generation quality, and strong generalization. Without any training, specific hardware, complex code adaptation, or additional runtime statistics, it achieves more than $10\times$ speedup while preserving generation quality. MrFlow can freely set different configurations to flexibly realize the trade-off between speedup and generation quality, and can not only outperform various training-free acceleration strategies but also be directly combined orthogonally with existing timestep distillation models to achieve even higher acceleration of more than $25\times$.
\end{itemize}

\section{Related Work}

\textbf{Flow matching.} Flow matching~\citep{Lipman2023FlowMatching} and rectified flow~\citep{Liu2023RectifiedFlow} recast diffusion training as learning a velocity field along linear interpolants between data and noise. For a data sample $x_0$ and noise $\epsilon\sim\mathcal{N}(0, I)$, the forward process is:
\begin{equation}
x_t = (1-\sigma_t)x_0 + \sigma_t\epsilon,
\end{equation}
and the model learns:
\begin{equation}
v_\theta(x_t, t)\approx\mathbb{E}[\epsilon - x_0 \mid x_t],
\end{equation}
sampling integrates $\dot x_t = v_\theta$ from $t{=}1$ to $t{=}0$. Diffusion models since SD3~\citep{SD3}, including typical advanced models such as FLUX~\citep{FLUX2024} and Qwen-Image~\citep{QwenImage2025}, almost all adopt this formulation.

\textbf{Acceleration of diffusion generation.} Quantization~\citep{Li2023QDiffusion} and efficient attention~\citep{Dao2022FlashAttention} are general compression and acceleration strategies that rely on hardware or systems to implement, and their acceleration gains on diffusion are relatively limited. The diffusion-specific acceleration methods are mainly timestep reduction~\citep{Salimans2022ProgressiveDistillation, Song2023ConsistencyModels}, feature caching~\citep{TeaCache2024, CacheDit}, and token pruning~\citep{ToMA}, all of which can improve sampling efficiency without hardware dependence. Accompanying a speedup of $10\times$ and above, strategies represented by timestep distillation~\citep{Luo2023LCM, SenseFlow, PiFlow} usually rely on fine-tuning training of non-negligible cost. Feature caching can typically attain around $4\times$ inference speedup without training. Existing works on token compression can usually only attain less than $1.5\times$ inference speedup.

\textbf{Multi-resolution generation.} Leveraging the characteristic information distribution that images naturally possess at different resolutions, many works have explored multi-stage generation from low resolution to high resolution~\citep{SR3, CascadedDiffusion, Imagen}. One class targets extending the attainable resolution~\citep{DemoFusion, MegaFusion}. The goal of this class of work is to generate larger images beyond the native resolution, and they do not concern themselves with acceleration gains at the native resolution. The second class treats multi-resolution as an acceleration means~\citep{LSSGen, RALU, BottleneckSampling, SPD}; they aim to place the generation focus on the low resolution and only infer a few steps at high resolution to achieve acceleration, and they all complete the upsampling within the latent or frequency domain.

\textbf{Super-resolution and image repainting.} Image super-resolution in the pixel space has long been divided into three classes: regression-based methods represented by SwinIR~\citep{SwinIR}, generative-adversarial-based methods represented by Real-ESRGAN~\citep{RealESRGAN}, and diffusion-based methods represented by OSEDiff~\citep{OSEDiff}. Image repainting is another line of work related to ours. Methods such as SDEdit~\citep{SDEdit} add noise on an existing image and then perform reverse denoising, in order to accomplish image editing while preserving semantics; this class of methods treats noising-denoising as the basic tool of image editing.

\section{Method}
\label{sec:method}

We propose MrFlow, a multi-resolution generation strategy designed from an acceleration perspective. Its tenet is to first rapidly generate the overall image structure at the low-resolution stage, then map back to the pixel space for super-resolution, subsequently add low-strength noise for high-frequency resampling, and finally make a small amount of detail refinement on the image at the high-resolution stage. Specifically, MrFlow comprises the following procedures: low-resolution latent-space sampling, VAE decoding, pixel-space super-resolution, VAE encoding, noise injection, high-resolution latent-space sampling, and VAE decoding. As a typical configuration, MrFlow requires only $12$ sampling steps at the low-resolution stage, the noise strength injected after super-resolution is usually only $0.1$, and then one denoising step at the high-resolution stage suffices. The overall design of the pipeline is built upon one pervasive observation: the low-resolution stage determines the global structure and semantic layout of the image, while the high-resolution stage only needs to refine the details of the super-resolved image from the aforementioned stage. This chapter elaborates the design of each stage of MrFlow; the quantitative analysis and ablation experiments are placed uniformly in Section~\ref{sec:main_results} and Appendix~\ref{app:stage_analysis}.

\subsection{Low-Resolution Structure Generation}
\label{sec:lr_generation}

\textbf{Formulation.} This stage first samples initial noise in the low-resolution latent space:
\begin{equation}
\mathbf{z}_1^{\mathrm{LR}} \sim \mathcal{N}(\mathbf{0},\mathbf{I}), \qquad \mathbf{z}_1^{\mathrm{LR}} \in \mathbb{R}^{C\times H_L\times W_L},
\end{equation}
where $C$ is the number of latent channels, and $H_L\times W_L$ is determined by the target output size through the downsampling factor of the pretrained VAE.

It then performs ODE sampling in the flow matching form to obtain the clean latent:
\begin{equation}
\mathbf{z}_0^{\mathrm{LR}} \;=\; \Phi_{\mathbf{v}_\theta,\,\mathbf{c}}^{K_L}\!\big(\mathbf{z}_1^{\mathrm{LR}}\big),
\end{equation}
where $\mathbf{v}_\theta$ is the pretrained flow matching velocity network, $\mathbf{c}$ is the text-condition embedding, and $\Phi_{\mathbf{v}_\theta,\mathbf{c}}^{K}$ denotes the rectified flow ODE flow map obtained by $K$-step Euler discretization, with the iteration rule $\mathbf{z}_{t-h}=\mathbf{z}_t-h\,\mathbf{v}_\theta(\mathbf{z}_t,t,\mathbf{c})$. This work typically takes $K_L=12$.

Finally, the clean latent is mapped back to the pixel space through the pretrained VAE decoder $\mathcal{D}$:
\begin{equation}
\mathbf{x}_{\mathrm{LR}} \;=\; \mathcal{D}\!\big(\mathbf{z}_0^{\mathrm{LR}}\big), \qquad \mathbf{x}_{\mathrm{LR}}\in\mathbb{R}^{3\times H_L^{\mathrm{px}}\times W_L^{\mathrm{px}}},
\end{equation}
where $H_L^{\mathrm{px}}\times W_L^{\mathrm{px}}$ is the low-resolution pixel-space size.

The low-resolution image generated at this stage already almost possesses the complete global structure and the semantic features that match the prompt. However, at the same display size, it usually appears blurrier compared with conventional images generated directly at high resolution, and local fine-grained semantics such as text and symbols are difficult to generate stably.

\textbf{Analysis.} The acceleration of the low-resolution stage comes from two mutually independent sources. The first is that the overall cost of single-step inference is measured to scale almost linearly with the number of image tokens, yielding about $4\times$ comprehensive speedup when each side is shrunk by $2\times$. The second is that the number of sampling steps required at low resolution is itself fewer: under an equal step-count budget, a two-stage schedule dominated by low resolution needs only $10$ low-resolution steps plus $1$ high-resolution step to attain even better visual quality than direct generation with $11$ high-resolution steps, with latency reduced by more than half relative to the latter while CLIP consistency is instead higher. We attribute this few-step convergence phenomenon to two complementary causes: the higher text-to-image semantic utilization at low resolution, and the shorter ODE path corresponding to the low-frequency skeleton; the quantitative decomposition and corresponding evidence are given in Appendix~\ref{app:c1}. Appendix~\ref{app:c5} further characterizes the role of this stage in establishing the global structure of the generated image.

\subsection{Pixel-Space Super-Resolution}
\label{sec:pixel_sr}

\begin{figure}[t]
\begin{center}
\includegraphics[width=\linewidth]{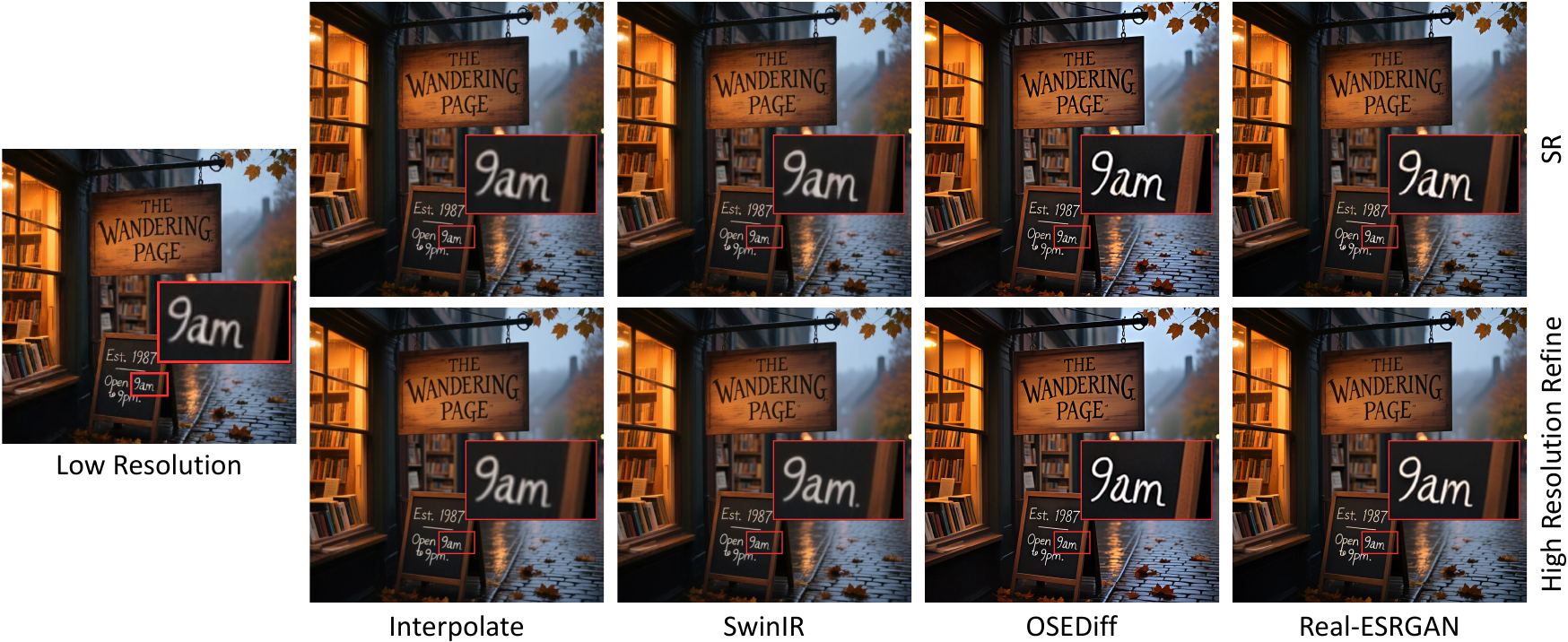}
\end{center}
\caption{Example comparison of MrFlow adopting different super-resolution strategies on Qwen-Image, where Low Resolution denotes the image after low-resolution generation followed by VAE decoding to the pixel space, with a resolution of $512\times512$. The rows of SR and High Resolution Refine are respectively the super-resolved image and the final image after high-resolution refinement, with a resolution of $1024\times1024$.}
\label{fig:sr_compare}
\end{figure}

\textbf{Formulation.} For the image obtained after low-resolution sampling and VAE decoding, we apply super-resolution in the pixel space:
\begin{equation}
\mathbf{x}_{\mathrm{SR}} \;\gets\; U\!\big(\mathbf{x}_{\mathrm{LR}}\big),\qquad \mathbf{x}_{\mathrm{SR}}\in\mathbb{R}^{3\times H_H^{\mathrm{px}}\times W_H^{\mathrm{px}}},
\end{equation}
where $U$ is the pretrained super-resolution network, upsampling the pixel-space image from the low-resolution size $H_L^{\mathrm{px}}\times W_L^{\mathrm{px}}$ to the high-resolution size $H_H^{\mathrm{px}}\times W_H^{\mathrm{px}}$, with $\times 2$ upsampling typically chosen in this work. Specifically, we directly adopt the pretrained Real-ESRGAN~\citep{RealESRGAN} as $U$, rather than simple interpolation or L2-supervised super-resolution models such as SwinIR~\citep{SwinIR}.

While enlarging the resolution, the super-resolution network both completely preserves the overall image structure generated at the low-resolution stage and supplements part of the high-frequency information. However, since the pretrained super-resolution network does not carry the semantic information in the prompt and originates only from pretrained knowledge of distribution matching between different resolutions in general scenes, the high-resolution image at this stage may still exhibit local structural confusion such as artifacts and character-stroke shifts.

\textbf{Analysis.} This stage involves two key design choices: performing upsampling in the pixel space, and adopting a GAN-based pretrained super-resolution model. First, the pixel space is a more robust upsampling domain: the edges, textures, and power-spectrum priors of natural images are all built upon pixels, and the pretrained knowledge of prior super-resolution literature can only be reused in the pixel domain; meanwhile, the VAE encoder naturally attenuates the high-frequency components in the pixel domain that exceed the training distribution, performing a lightweight regularization on the super-resolution output; in addition, the extra VAE encoding-decoding cost is very small relative to the inference cost of diffusion. Second, the downstream refine is effective only when the SR output keeps the coarse layout and leaves mainly local high-frequency residuals. GAN-based super-resolution better matches this regime by producing sharp, natural-image-like details, whereas interpolation or regression-based SR preserves the layout but leaves persistent high-frequency attenuation and diffuse blur. Appendix~\ref{app:c2} gives the corresponding empirical comparison.

\subsection{Low-Strength Noise for High-Frequency Resampling}
\label{sec:noise_resampling}

\textbf{Formulation.} After the pixel-space super-resolution, re-encoding through the pretrained VAE encoder $\mathcal{E}$ yields the high-resolution latent:
\begin{equation}
\mathbf{z}_0^{\mathrm{SR}} \;=\; \mathcal{E}\!\big(\mathbf{x}_{\mathrm{SR}}\big),\qquad \mathbf{z}_0^{\mathrm{SR}} \in \mathbb{R}^{C\times H_H\times W_H}.
\end{equation}
Unlike the clean latent obtained by direct high-resolution sampling from noise, $\mathbf{z}_0^{\mathrm{SR}}$ may contain artifacts and character-stroke shifts introduced by GAN super-resolution. These errors are mostly localized in high-frequency directions, while the low-frequency structure is already determined by the low-resolution stage and should be preserved.

To this end, before entering the high-resolution refine, we inject low-strength noise in the flow matching form to obtain the noised SR latent:
\begin{equation}
\mathbf{z}_t^{\mathrm{HR}} \;=\; (1-\sigma_t)\,\mathbf{z}_0^{\mathrm{SR}} \,+\, \sigma_t\,\boldsymbol{\epsilon},\qquad \boldsymbol{\epsilon}\sim\mathcal{N}(\mathbf{0},\mathbf{I}),
\end{equation}
where $\sigma_t\in(0,1)$ is the flow noise level matched to the refinement timestep, and this work typically takes $\sigma_t\in[0.1,0.15]$. This step keeps the low-frequency structure at high SNR, while lowering the SNR of high-frequency directions so that the following denoising step can resample them according to the high-resolution flow prior.

\textbf{Analysis.} The global structure and low-frequency content of $\mathbf{z}_0^{\mathrm{SR}}$ are largely inherited from the low-resolution stage and should be preserved, while the remaining uncertainty mainly lies in the high-frequency details introduced by GAN super-resolution. Therefore, the injected noise only needs to reduce the high-frequency SNR enough for denoising to resample that band according to the data prior and correct the potential errors of the SR. Let $\lambda_{\mathrm{hf}}$ be the power, i.e., the variance, of the clean latent's high-frequency band; the signal-to-noise-ratio condition yields the lower bound of the noise level:
\begin{equation}
\sigma_t \;\ge\; \sigma_t^\star=\frac{\sqrt{\lambda_{\mathrm{hf}}}}{1+\sqrt{\lambda_{\mathrm{hf}}}}.
\end{equation}

This lower bound is computed from the measurable high-frequency power of the clean data, rather than from an unknown norm of the SR error. It therefore characterizes the frequency regime where low-strength noising is effective: once the SR output is locally plausible and its residual errors are mainly high-frequency, the small high-frequency power of the clean latent makes $\sigma_t\in[0.1,0.15]$ sufficient for resampling. If the SR output instead contains diffuse blur or low-frequency deviations, this premise no longer holds and substantially stronger noising would be needed, which would wash away useful SR details and require more steps. The complete posterior-mean derivation, the derivation of the above inequality, the supplementary distribution-level view, and the power-spectrum measurement are given in Appendix~\ref{app:c3}.

\subsection{High-Resolution Detail Refinement}
\label{sec:hr_refinement}

\textbf{Formulation.} After completing the low-strength noising on the high-resolution latent, we take the noised high-resolution latent $\mathbf{z}_t^{\mathrm{HR}}$ as the initial value and apply the flow map on the high resolution along the same velocity-estimation network:
\begin{equation}
\mathbf{z}_0^{\mathrm{HR}} \;=\; \Phi_{\mathbf{v}_\theta,\,\mathbf{c}}^{K_H}\!\big(\mathbf{z}_t^{\mathrm{HR}}\big),
\end{equation}
where $\mathbf{c}$ is the text-condition embedding completely identical to that of the low-resolution stage, and $K_H$ is the number of Euler steps at the high-resolution stage. In practice, we start the high-resolution flow from the timestep whose noise level matches the preceding noising stage, and use $K_H=1$ by default; additional high-resolution steps bring only marginal changes under this low-strength regime. Finally, the clean latent is sent back to the pixel space through the VAE decoder $\mathcal{D}$ to obtain the final output:
\begin{equation}
\mathbf{x}_{\mathrm{HR}} \;=\; \mathcal{D}\!\big(\mathbf{z}_0^{\mathrm{HR}}\big),\qquad \mathbf{x}_{\mathrm{HR}}\in\mathbb{R}^{3\times H_H^{\mathrm{px}}\times W_H^{\mathrm{px}}}.
\end{equation}
This stage performs the final one-step detail refinement on the local artifacts introduced by the SR output, and the entire pipeline at this point outputs a high-quality image at the target resolution.

\begin{wrapfigure}{r}{0.6\linewidth}
\begin{center}
\vspace{-0.2in}
\includegraphics[width=\linewidth]{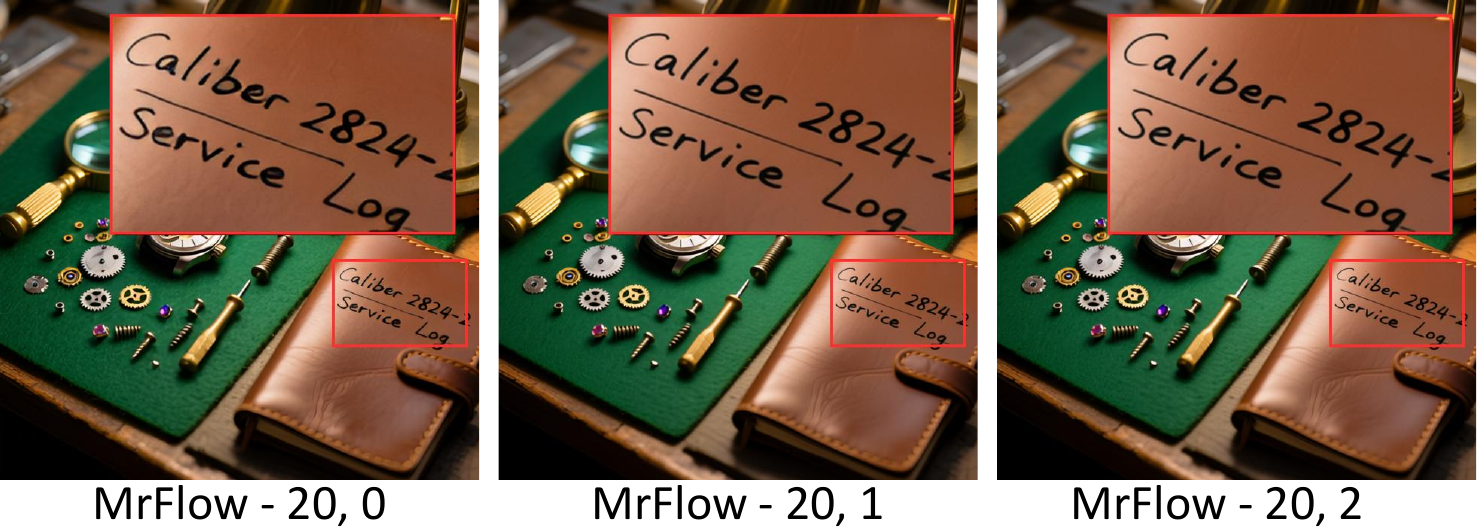}
\end{center}
\caption{Images generated by MrFlow under different high-resolution step configurations on Qwen-Image.}
\label{fig:mrflow-compare-qwen}
\end{wrapfigure}

\textbf{Analysis.} The reason the high-resolution refine stage can be completed with a single step is that the noised SR latent is already close to the clean endpoint. Under low-strength noising, the velocity-field magnitude is flatter and the adjacent-step difference is smaller, so the single-step Euler discretization error is correspondingly lower; in measurements, single-step denoising at $s=0.1$ already attains a CLIP consistency of $0.9974$ relative to the $8$-step reference, only $0.0025$ away from the $0.9999$ of $5$-step denoising. This phenomenon is consistent with the observation in SD3~\citep{SD3} timestep shift that the schedule segments near the clean endpoint have low cost, and can be regarded as a special case of that principle extrapolated to SR followed by low-strength noise injection. The velocity-field characterization of this single-step behavior and the ablation of schedule shape are given in Appendix~\ref{app:c4}; Appendix~\ref{app:c5} further characterizes this stage as a detail-refinement step that preserves the established structure.

\section{Empirical Results}

We mainly conduct experiments on text-to-image tasks. Section~\ref{sec:impl} reviews the experimental setup, Section~\ref{sec:main_results} mainly presents extensive comparisons with various acceleration strategies on quantitative metrics, and Section~\ref{sec:ablation} is the overall ablation experiment of MrFlow. Further analysis of each stage of MrFlow is detailed in Appendix~\ref{app:stage_analysis}, and more visualization examples are detailed in Appendix~\ref{app:examples}.

\subsection{Implementation Details}
\label{sec:impl}

Our main acceleration experiments are conducted on FLUX.1-dev~\citep{FLUX2024} and Qwen-Image-20B~\citep{QwenImage2025}, with quantitative evaluation uniformly set at $1024\times1024$ resolution, and the evaluation metrics adopt three classes: Geneval~\citep{Geneval}, DPG-Bench~\citep{DPGBench}, and OneIG-Bench~\citep{OneIGBench}. The selected comparison methods include typical training-free acceleration strategies of various classes, comprising the feature-cache-based Teacache~\citep{TeaCache2024} and DB-Taylor (the comprehensive version fusing DBCache and TaylorSeer~\citep{TaylorSeer} in Cache-Dit~\citep{CacheDit}, which we abbreviate as DB-Taylor), the token-fusion-based ToMA~\citep{ToMA}, and LSSGen~\citep{LSSGen}, RALU~\citep{RALU}, and SPEED~\citep{SPD} which are also based on multi-resolution. For extremely high speedup scenarios greater than $8\times$, given that the above training-free methods essentially all face collapse, we additionally compare the state-of-the-art training-dependent timestep distillation methods, SenseFlow~\citep{SenseFlow} and Pi-Flow~\citep{PiFlow}, uniformly set to the $4$-step inference version. For the training-required strategies, including all timestep distillation methods and LSSGen, we directly load the parameter weights provided in their official repositories without retraining. When MrFlow is combined with the weights pre-trained by timestep distillation, we denote it as MrFlow$^\dagger$; this strategy directly loads the distilled model at both the low-resolution and high-resolution stages, without any additional training in combination with the multi-resolution pipeline. Each experiment is run on a single Nvidia A100 GPU. The reported speedups are all end-to-end actual acceleration, including text encoding, random noise generation, VAE encoding/decoding, SR, and all diffusion forward passes. All detailed configurations are given in Appendix~\ref{app:setup}.

\subsection{Main Results}
\label{sec:main_results}

In this section we conduct quantitative tests across various acceleration strategies with different characteristics.

\begin{table}[t]
\caption{Comparison of various training-free acceleration methods on FLUX.1-dev and Qwen-Image, generated at $1024\times1024$ resolution.}
\vspace{-0.05in}
\label{tab:flux_tf}
\label{tab:qwen_tf}
\begin{center}
\setlength{\tabcolsep}{0.8mm}
\begin{tabular}{lccccccc}
\toprule
Method & Training-free & NFEs & Speedup $\uparrow$ & Geneval $\uparrow$ & DPG $\uparrow$ & OneIG-En $\uparrow$ & OneIG-Zh $\uparrow$ \\
\midrule
FLUX.1-dev  & -        & 50    & 1x    & 0.66 & 84.07 & 0.44 & - \\
\midrule
ToMA      & $\surd$  & 50    & 1.05x & 0.66 & 83.99 & 0.42 & - \\
ToMA      & $\surd$  & 50    & 1.13x & 0.66 & 83.78 & 0.21 & - \\
\cdashline{1-8}
Teacache  & $\surd$  & 50    & 4.47x & 0.63 & 82.62 & 0.36 & - \\
DB-Taylor & $\surd$  & 50    & 4.63x & 0.64 & \textbf{83.78} & \textbf{0.40} & - \\
RALU      & $\surd$  & 4, 5, 6 & 4.33x & 0.62 & 82.83 & 0.32 & - \\
SPEED     & $\surd$  & 3, 2, 7 & 5.71x & 0.62 & 83.55 & 0.37 & - \\
\textbf{MrFlow}    & \textbf{$\surd$}  & 20, 1 & \textbf{5.78x} & \textbf{0.65} & 82.19 & 0.39 & - \\
\cdashline{1-8}
Teacache  & $\surd$  & 50    & 7.57x & 0.45 & 70.79 & 0.28 & - \\
DB-Taylor & $\surd$  & 50    & 7.86x & 0.46 & 74.23 & 0.28 & - \\
RALU      & $\surd$  & 1, 2, 3 & 8.21x & 0.53 & 77.11 & 0.30 & - \\
SPEED     & $\surd$  & 2, 1, 5 & 8.01x & 0.62 & \textbf{83.43} & 0.36 & - \\
\textbf{MrFlow}    & \textbf{$\surd$}  & 12, 1 & \textbf{8.25x} & \textbf{0.63} & 81.65 & \textbf{0.36} & - \\
\midrule
Qwen-Image & -       & 50x2 & 1x    & 0.88 & 88.67 & 0.53 & 0.53 \\
\midrule
Teacache   & $\surd$ & 50x2 & 4.61x & 0.81 & 84.46 & 0.42 & 0.45 \\
DB-Taylor  & $\surd$ & 50x2 & 4.13x & 0.86 & 86.81 & 0.50 & 0.50 \\
RALU       & $\surd$ & (4, 5, 5)x2 & 5.85x & 0.85 & 87.07 & 0.46 & 0.48 \\
SPEED      & $\surd$ & (2, 1, 7)x2 & 6.29x & 0.79 & 87.50 & 0.47 & 0.51 \\
\textbf{MrFlow}     & \textbf{$\surd$} & (20, 1)x2 & \textbf{6.98x} & \textbf{0.87} & \textbf{88.00} & \textbf{0.54} & \textbf{0.52} \\
\cdashline{1-8}
Teacache   & $\surd$ & 50x2 & 8.93x & 0.26 & 32.51 & 0.16 & 0.19 \\
DB-Taylor  & $\surd$ & 50x2 & 9.34x & 0.09 & 17.43 & 0.09 & 0.12 \\
RALU       & $\surd$ & (1, 2, 4)x2 & 9.51x & 0.84 & 82.60 & 0.44 & 0.46 \\
SPEED      & $\surd$ & (1, 1, 5)x2 & 8.61x & 0.74 & 84.93 & 0.44 & 0.46 \\
\textbf{MrFlow}     & \textbf{$\surd$} & (12, 1)x2 & \textbf{10.3x} & \textbf{0.86} & \textbf{87.10} & \textbf{0.52} & \textbf{0.51} \\
\bottomrule
\end{tabular}
\\[2pt]
\end{center}
\end{table}

\textbf{Training-free acceleration.} All methods natively support FLUX.1-dev, while the native version of Qwen-Image does not adopt the distilled single-branch CFG, so every step needs two forward passes. We mark whether each method requires training, and use dashed lines to separate different levels of speedup. Table~\ref{tab:flux_tf} lists the results of various training-free methods on the two backbones. ToMA, as a representative method based on token pruning and merging, can only attain a very limited speedup of less than $1.1\times$ in measurements. And although its metric drop is not noticeable, the generated images have actually collapsed significantly. The other multi-resolution-based method, RALU, already exhibits a certain accuracy degradation at a speedup of around $5\times$ on FLUX, and when the speedup is further raised to around $8\times$ the generation quality drops even more noticeably; although RALU can better preserve accuracy at high speedup on Qwen-Image, its metrics also drop noticeably when the speedup approaches $10\times$, and the actual image quality is even worse under visual inspection. SPEED provides a stronger frequency-domain multi-resolution baseline than earlier latent-upsampling strategies, but its accuracy drops sharply on Qwen-Image when pushed to the same speed range as MrFlow. Teacache and DB-Taylor likewise preserve good generation results at a speedup of around $4\times$, but when the speedup is further raised to $8\times$, they both exhibit significant collapse; on the newer, larger-scale, and more capable Qwen-Image, the feature-cache-class methods exhibit even larger losses, with a gap of more than $3\%$ already at a speedup of $4$-$5\times$, and direct collapse at the more aggressive speedup of $9\times$. MrFlow, however, demonstrates a greater advantage: it preserves good accuracy under the configuration of $20$ low-resolution steps and $1$ high-resolution step, with a gap of only about $2\%$ on Geneval relative to the native $50$ steps. When set to $12$ low-resolution steps and $1$ high-resolution step, the speedup is raised to $8.25\times$ on FLUX and $10.3\times$ on Qwen-Image, while Geneval remains at $0.63$ and $0.86$, respectively, significantly surpassing feature-cache-based methods such as DB-Taylor and multi-resolution-based strategies such as RALU and SPEED under aggressive configurations.

\textbf{Training-dependent methods.} Table~\ref{tab:flux_distill} further compares MrFlow with training-dependent acceleration methods on the two backbones, including state-of-the-art timestep distillation methods as well as the multi-resolution method LSSGen, which requires training a lightweight latent upsampler and is therefore grouped with the training-dependent strategies rather than the training-free ones. LSSGen attains a $1.5\times$ speedup while preserving accuracy under the officially suggested default configuration with a noise strength of $0.75$; however, when the noising strength is lowered to $0.2$ for a $3.9\times$ speedup, the quality of the actually generated images has already dropped substantially, as detailed in Appendix~\ref{app:d1}. Here $\dagger$ or $\ddagger$ both indicate that the combination of MrFlow with the distillation method directly utilizes the pre-trained distillation weights without additional training in combination with MrFlow, and the two respectively denote the combination of MrFlow with Pi-Flow and with FLUX-schnell~\citep{FLUX2024}. On FLUX, MrFlow under the training-free $12{+}1$-step setting already attains accuracy close to the $4$-step SenseFlow, while MrFlow+Pi-Flow attains a higher speedup than the native $4$-step Pi-Flow while maintaining comparable Geneval. When MrFlow is combined with FLUX-schnell, although the speedup is on par with the native $2$-step FLUX-schnell due to extra overhead such as VAE, DPG is also slightly higher than the native FLUX-schnell. Since timestep distillation-class methods can directly fuse the dual-branch CFG mechanism of Qwen-Image into a single forward, they themselves naturally possess a higher speedup; on Qwen-Image, combining MrFlow with Pi-Flow follows the same trend as on FLUX and reaches up to $25\times$ speedup with an OneIG-Bench loss of no more than $1\%$. These results indicate that MrFlow under the training-free setting can already attain accuracy and efficiency comparable to training-dependent timestep distillation strategies, and has good combinability with timestep distillation strategies.

\begin{table}[t]
\caption{Comparison with training-dependent acceleration methods on FLUX.1-dev and Qwen-Image, generated at $1024\times1024$ resolution.}
\vspace{-0.05in}
\label{tab:flux_distill}
\label{tab:qwen_distill}
\begin{center}
\setlength{\tabcolsep}{1.06mm}
\begin{tabular}{lccccccc}
\toprule
Method & Training-free & NFEs & Speedup $\uparrow$ & Geneval $\uparrow$ & DPG $\uparrow$ & OneIG-En $\uparrow$ & OneIG-Zh $\uparrow$ \\
\midrule
FLUX.1-dev          & -          & 50    & 1x    & 0.66 & 84.07 & 0.44 & - \\
\midrule
LSSGen            & $\times$   & 25, 25 & 1.49x & 0.65 & 83.51 & 0.40 & - \\
LSSGen            & $\times$   & 25, 4  & 3.93x & 0.64 & 82.77 & 0.40 & - \\
\cdashline{1-8}
\textbf{MrFlow}            & \textbf{$\surd$}    & 12, 1 & 8.25x & 0.63 & 81.08 & 0.36 & - \\
SenseFlow         & $\times$   & 4     & 11.1x & 0.63 & 82.42 & 0.37 & - \\
Pi-Flow           & $\times$   & 4     & 9.49x & 0.68 & \textbf{84.40} & \textbf{0.41} & - \\
\textbf{MrFlow$^\dagger$}  & \textbf{$\dagger$}   & 4, 1  & \textbf{11.3x} & \textbf{0.69} & 83.26 & 0.39 & - \\
\cdashline{1-8}
FLUX-schnell       & $\times$   & 2     & \textbf{20.4x} & 0.69 & 85.03 & \textbf{0.41} & - \\
\textbf{MrFlow$^\ddagger$} & \textbf{$\ddagger$} & 1, 1  & 19.5x & \textbf{0.71} & \textbf{85.10} & 0.35 & - \\
\midrule
Qwen-Image       & -        & 50x2 & 1x    & 0.88 & 88.67 & 0.53 & 0.53 \\
\midrule
Pi-Flow          & $\times$ & 4    & 19.6x & \textbf{0.85} & \textbf{88.10} & \textbf{0.53} & \textbf{0.52} \\
\textbf{MrFlow$^\dagger$} & \textbf{$\dagger$} & 4, 1 & \textbf{25.1x} & \textbf{0.85} & 87.52 & 0.52 & 0.51 \\
\bottomrule
\end{tabular}
\\[2pt]
\end{center}
\end{table}

\subsection{Ablation Study}
\label{sec:ablation}

\textbf{Different step configurations.} On Qwen-Image, we conduct combinatorial experiments over all training-free methods and over MrFlow configurations with the low-resolution stage set to $10$, $12$, $16$, $20$ steps and the high-resolution stage set to $1$, $2$, $3$ steps; the trade-off result between Geneval and speedup is shown in Figure~\ref{fig:tradeoff}. Evidently, MrFlow is the only method that consistently outperforms across different models the naive acceleration strategy of directly applying different timestep configurations to the native model, and it can flexibly be configured toward either better accuracy or a more aggressive speedup. Note that MrFlow's performance on test metrics improves significantly as the number of low-resolution steps increases, but has no obvious connection with the high-resolution step configuration. The visual comparison leads to the same conclusion. As shown in Figure~\ref{fig:mrflow-compare-qwen}, adding one high-resolution refinement step already makes the text strokes clear and removes the blur observed in the no-refinement output, while using two refinement steps produces a highly similar result. This agrees with the analysis in Section~\ref{sec:hr_refinement}: after low-strength noising, the trajectory segment close to the clean-image endpoint is sufficiently straight, so a single high-resolution step is an efficient default for detail refinement.

\begin{figure}[t]
\begin{center}
\includegraphics[width=\linewidth]{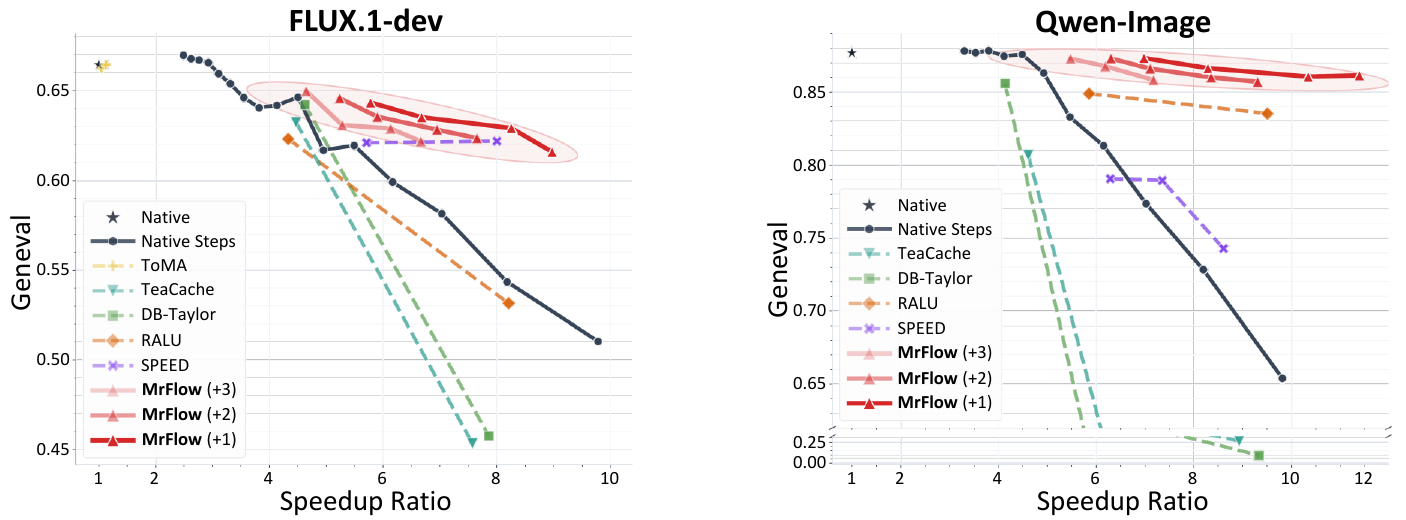}
\end{center}
\caption{Trade-off curves between GenEval score and speedup for various training-free methods on FLUX.1-dev and Qwen-Image. Different shades of MrFlow correspond to different numbers of steps used in the high-resolution refinement stage.}
\label{fig:tradeoff}
\end{figure}

\textbf{Super-resolution network.} On Qwen-Image, we compare several SR choices: direct interpolation, SwinIR, OSEDiff, Real-ESRGAN, and SR-only output without HR refinement. As shown in Table~\ref{tab:sr_ablation}, most SR choices stay close on automatic metrics, while OSEDiff is slightly below Real-ESRGAN. However, the visual comparison in Figure~\ref{fig:sr_compare}, generated with $12$ LR steps and $1$ HR step, shows that these metrics are not always sensitive to local high-frequency SR defects: HR refinement corrects local deviations such as character strokes and text-like details while preserving global structure. Interpolate and SwinIR remain blurry, OSEDiff introduces character inaccuracies after enlargement, and GAN-based Real-ESRGAN offers the best balance between clarity, semantic accuracy, and efficiency.

\begin{table}[t]
\caption{Comparison of different super-resolution method choices on Qwen-Image, generated at $1024\times1024$ resolution.}
\label{tab:sr_ablation}
\begin{center}
\setlength{\tabcolsep}{2.3mm}
\begin{tabular}{lcccccc}
\toprule
SR method & NFEs & Speedup $\uparrow$ & Geneval & DPG $\uparrow$ & OneIG-En $\uparrow$ & OneIG-Zh $\uparrow$ \\
\midrule
Interpolate & (12, 1)x2 & 10.7x & 0.87 & 87.00 & 0.52 & 0.51 \\
SwinIR      & (12, 1)x2 & 4.67x & 0.86 & 86.97 & 0.52 & 0.51 \\
OSEDiff     & (12, 1)x2 & 9.45x & 0.85 & 87.25 & 0.51 & 0.49 \\
Real-ESRGAN & (12, 1)x2 & 10.3x & 0.86 & 87.10 & 0.52 & 0.51 \\
Real-ESRGAN & 12x2      & 14.0x & 0.87 & 87.45 & 0.52 & 0.51 \\
\bottomrule
\end{tabular}
\end{center}
\end{table}

\section{Conclusion}

We propose MrFlow, a training-free multi-resolution strategy for accelerating pretrained flow-matching models. MrFlow combines fast low-resolution structure generation, pixel-space super-resolution with a lightweight GAN-based network, low-strength latent noising, and one-step high-resolution refinement. Without any training or runtime dynamic statistics, this staged pipeline achieves more than $10\times$ end-to-end speedup while keeping the OneIG-Bench loss within $1\%$ relative to native inference, and it offers better quality, efficiency, flexibility, and simplicity than other training-free diffusion acceleration strategies. MrFlow can also be directly combined with timestep distillation without additional training, yielding a further compounded speedup of $25\times$.

\bibliographystyle{iclr2026_conference}
\bibliography{iclr2026_conference}

\newpage

\appendix

\section{Background}
\label{app:background}

\textbf{Acceleration of diffusion generation.} Quantization~\citep{Li2023QDiffusion, zheng2025binarydm, zheng2026first} and efficient attention~\citep{Dao2022FlashAttention} are compression and acceleration strategies oriented to general structures, and their acceleration usually relies on hardware or systems to implement. The diffusion-specific acceleration methods are mainly timestep reduction, feature caching, and token pruning, all of which can improve sampling efficiency without hardware dependence. A major direction for accelerating diffusion models is reducing the number of denoising steps. Existing methods use improved samplers~\citep{Song2021DDIM, Lu2022DPMSolver} or distillation techniques~\citep{Salimans2022ProgressiveDistillation, Song2023ConsistencyModels, Luo2023LCM} to compress a multi-step teacher into a few-step or even one-step generator. Accompanying a speedup of $10\times$ and above, strategies represented by timestep distillation~\citep{SenseFlow, PiFlow} usually rely on fine-tuning training of non-negligible cost, and the trajectories are likely to gradually converge to losing the diversity across different random seeds. Another line of work accelerates diffusion models through feature caching~\citep{TeaCache2024, CacheDit, TaylorSeer, feng2026worldcache}. These methods reuse intermediate features across adjacent denoising steps based on the observation that diffusion features change slowly over time, thereby avoiding repeated computation in the U-Net or DiT backbone. These approaches typically achieve around $2\times$ to $4\times$ inference speedup with relatively small quality degradation. Token compression~\citep{ToMA, lin2026token} reduces computation by dynamically removing or fusing spatial tokens based on their estimated importance during denoising. Existing approaches typically achieve only less than $1.5\times$ inference speedup. These diffusion acceleration strategies either rely on fine-tuning training of non-negligible cost, or can only attain limited speedup.

\textbf{Multi-resolution generation.} Leveraging the characteristic information distribution that images naturally possess at different resolutions, many works have explored multi-stage generation from low resolution to high resolution. This line can be roughly divided into two classes. The first class targets extending the attainable resolution. Early works such as SR3~\citep{SR3}, Cascaded Diffusion Models~\citep{CascadedDiffusion}, and Imagen~\citep{Imagen} adopt cascaded diffusion or diffusion-based super-resolution, generating the main structure at low resolution and then progressively enlarging and supplementing details; Hires.\ fix~\citep{HiresFix} further validated in community practice the feasibility of ``low-resolution generation + high-resolution repainting''; works such as DemoFusion~\citep{DemoFusion} and MegaFusion~\citep{MegaFusion} break through the pretrained resolution ceiling through progressive denoising and multi-scale latent fusion. The goal of this class of work is to generate larger images beyond the native resolution, and they do not concern themselves with acceleration gains at the native resolution. The second class treats multi-resolution as an acceleration means. LSSGen~\citep{LSSGen} trains a lightweight latent upsampler and advocates direct upsampling in the latent space to avoid the artifacts introduced by VAE re-encoding after pixel-space upsampling. Bottleneck Sampling~\citep{BottleneckSampling} adopts interpolation operators such as Lanczos in the latent space, and uses a ``high-low-high'' U-shaped schedule together with timestep rescheduling to achieve training-free acceleration. RALU~\citep{RALU} likewise does nearest-neighbor interpolation in the latent space, but suppresses aliasing artifacts through region-adaptive early-edge upsampling and noise-timestep distribution matching. Fresco~\citep{Fresco} uses a Hadamard orthogonal transform to expand a single parent token into four child tokens, and maintains consistency across resolution stages with a global noise field. SPD~\citep{SPD} changes the upsampling space from the spatial domain to the frequency domain, progressively expanding by frequency band along the denoising trajectory. The above acceleration-class works all complete the upsampling within the latent or frequency domain.

\textbf{Super-resolution and image repainting.} Image super-resolution in the pixel space has long been divided into three classes: regression-based methods, generative-adversarial-based methods, and diffusion-based methods. Methods such as SwinIR~\citep{SwinIR}, based on a Swin Transformer~\citep{SwinTransformer} or convolutional backbone, take L1/L2 reconstruction loss as the main training objective, modeling the conditional mean of the high-resolution image given the low-resolution input. Real-ESRGAN~\citep{RealESRGAN}, building upon ESRGAN~\citep{ESRGAN}, introduces second-order degradation modeling and replaces pure pixel regression with discriminator loss and perceptual loss, aligning the output distribution with the natural-image distribution in the discriminative sense. Diffusion-based super-resolution is represented by OSEDiff~\citep{OSEDiff}, which performs single-step distillation on a pretrained diffusion and introduces the super-resolution input as a condition. Image repainting is another related line. SDEdit~\citep{SDEdit} proposes adding noise on an existing image and then performing reverse denoising, in order to accomplish local editing while preserving semantics; subsequent works such as InstructPix2Pix~\citep{InstructPix2Pix}, Imagic~\citep{Imagic}, and Prompt-to-Prompt~\citep{PromptToPrompt} extend the conditional guidance to various degrees. This class of methods treats noising-denoising as the basic tool of image editing. Unlike these editing methods, MrFlow does not use noising-denoising to preserve and edit an existing image; instead, it uses low-strength noise as a frequency-selective mechanism that preserves the generated low-frequency structure while enabling high-frequency resampling.

\section{Detailed Experimental Setup}
\label{app:setup}

Our main acceleration experiments are conducted on FLUX.1-dev~\citep{FLUX2024} and Qwen-Image-20B~\citep{QwenImage2025}, with some extended experiments on FLUX.2 Klein~\citep{flux-2-2025}, Z-Image and Z-Image-Turbo~\citep{ZImage}. The comparison methods cover typical acceleration strategies of various classes: the feature-cache-based TeaCache~\citep{TeaCache2024} and DB-Taylor (the comprehensive version fusing DBCache and TaylorSeer~\citep{TaylorSeer} in Cache-Dit~\citep{CacheDit}), the token-fusion-based ToMA~\citep{ToMA}, and the multi-resolution LSSGen~\citep{LSSGen}, RALU~\citep{RALU}, and SPEED~\citep{SPD}. For the extreme regime beyond $8\times$, where the training-free methods essentially all collapse, we additionally compare the state-of-the-art timestep distillation methods SenseFlow~\citep{SenseFlow} and Pi-Flow~\citep{PiFlow}, both at their released $4$-step setting. For all training-dependent strategies, i.e.\ the timestep distillation methods and LSSGen, we directly load the official pre-trained weights without retraining. The native models are run with their default parameters, e.g.\ a guidance scale of $3.5$ for FLUX.1-dev and CFG with guidance scale $4.0$ for Qwen-Image. For each baseline we report two operating points, chosen so that TeaCache, DB-Taylor, RALU, and SPEED cover comparable speed regimes; the per-method hyper-parameters at each point are listed in Tables~\ref{tab:flux_method_settings} and~\ref{tab:qwen_method_settings}. When MrFlow is combined with timestep-distillation weights, denoted MrFlow$^\dagger$, it directly loads the distilled model at both stages without any additional training.

\begin{table*}[t]
\caption{Parameter settings on FLUX.1-dev used in the main comparison tables. $N_{\mathrm{LR}}$ and $N_{\mathrm{ref}}$ denote the low-resolution denoising and high-resolution refinement steps of MrFlow.}
\label{tab:flux_method_settings}
\begin{center}
\footnotesize
\setlength{\tabcolsep}{4pt}
\renewcommand{\arraystretch}{1.2}
\begin{tabular}{@{}>{\raggedright\arraybackslash}p{2.5cm}
                   >{\centering\arraybackslash}p{2.1cm}
                   >{\raggedright\arraybackslash}p{8.1cm}@{}}
\toprule
\textbf{Method} & \textbf{NFEs} & \textbf{Parameter setting} \\
\midrule
\multicolumn{3}{@{}l}{\textit{\textbf{FLUX.1-dev} — training-free acceleration}} \\
\midrule
ToMA & 50 &
Token-merging ratio $r\in\{0.5,0.8\}$; destination budget over $M{=}256$ local tiles, merging every $\Delta t{=}3$ steps. \\
TeaCache & 50 &
Triggered by the accumulated relative $\ell_1$-change criterion; thresholds $\tau\in\{1.0,2.5\}$ for the two speed regimes. \\
DB-Taylor & 50 &
DBCache threshold $\delta\in\{0.30,0.60\}$, warm-up $K\in\{4,1\}$; TaylorSeer order $p{=}1$, $(F_n,B_n){=}(1,0)$. \\
RALU & $4,5,6$ / $1,2,3$ &
Three-stage schedule $\mathbf{N}{=}(4,5,6)$ or $(1,2,3)$; transitions $\mathbf{e}{=}(0.30,0.45,1.0)$ or $(0.20,0.30,1.0)$; HR ratio $\rho{=}0.20$ or $0.075$. \\
SPEED & $3,2,7$ / $2,1,5$ &
Official s3 setting with total steps $12/8$, scales $(0.25,0.5,1.0)$, and $\delta{=}0.01$; discrete schedules are derived from the power-spectrum transitions. \\
MrFlow & $20,1$ / $12,1$ &
Pixel-space staged sampling with $(N_{\mathrm{LR}},N_{\mathrm{ref}}){=}(20,1)$ or $(12,1)$; Real-ESRGAN $\times2$ before refinement. \\
\midrule
\multicolumn{3}{@{}l}{\textit{\textbf{FLUX.1-dev} — training-dependent acceleration}} \\
\midrule
LSSGen & $25,25$ / $25,4$ &
Trained latent upsampler; LR stage $25$ steps, HR stage from noise level $\sigma_s\in\{0.75,0.2\}$, giving $25$ and $4$ HR steps. \\
SenseFlow & 4 & Released $4$-step FLUX student model. \\
Pi-Flow & 4 &
Released pi-FLUX $4$-step policy model, FlowMap schedule as in the official adapter. \\
MrFlow\,+\,Pi-Flow & $4,1$ &
LR Pi-Flow $4$ steps, then Real-ESRGAN $\times2$ and one HR refinement step. \\
FLUX-schnell & 2 & Released FLUX.1-schnell $2$-step model. \\
MrFlow\,+\,FLUX-schnell & $1,1$ &
LR FLUX-schnell $1$ step, then Real-ESRGAN $\times2$ and one HR FLUX-schnell step. \\
\bottomrule
\end{tabular}
\end{center}
\end{table*}

\begin{table*}[t]
\caption{Parameter settings on Qwen-Image used in the main comparison tables. ``$\times2$'' marks true-CFG evaluation of the positive and negative branches.}
\label{tab:qwen_method_settings}
\begin{center}
\footnotesize
\setlength{\tabcolsep}{4pt}
\renewcommand{\arraystretch}{1.2}
\begin{tabular}{@{}>{\raggedright\arraybackslash}p{2.5cm}
                   >{\centering\arraybackslash}p{2.1cm}
                   >{\raggedright\arraybackslash}p{8.1cm}@{}}
\toprule
\textbf{Method} & \textbf{NFEs} & \textbf{Parameter setting} \\
\midrule
\multicolumn{3}{@{}l}{\textit{\textbf{Qwen-Image} — training-free acceleration}} \\
\midrule
TeaCache & $50\times2$ &
Accumulated relative $\ell_1$-change criterion; thresholds $\tau\in\{0.75,12.0\}$ for the two speed regimes. \\
DB-Taylor & $50\times2$ &
DBCache threshold $\delta\in\{0.35,1.0\}$, warm-up $K\in\{2,0\}$; TaylorSeer $p{=}1$, $(F_n,B_n){=}(1,0)$ or $(5,0)$, separate CFG branches. \\
RALU & $(4,5,5)\times2$ / $(1,2,4)\times2$ &
Three-stage schedule $\mathbf{N}{=}(4,5,5)$ or $(1,2,4)$; transitions $\mathbf{e}{=}(0.30,0.45,1.0)$ or $(0.20,0.30,1.0)$; HR ratio $\rho{=}0.20$ or $0.10$. \\
SPEED & $(2,1,7)\times2$ / $(1,1,5)\times2$ &
S3 setting with total steps $10/7$, scales $(0.25,0.5,1.0)$, and $\delta{=}0.01$; discrete schedules are derived from the Qwen shifted scheduler. \\
MrFlow & $(20,1)\times2$ / $(12,1)\times2$ &
Pixel-space staged true-CFG sampling with $(N_{\mathrm{LR}},N_{\mathrm{ref}}){=}(20,1)$ or $(12,1)$; Real-ESRGAN $\times2$ before refinement. \\
\midrule
\multicolumn{3}{@{}l}{\textit{\textbf{Qwen-Image} — training-dependent acceleration}} \\
\midrule
Pi-Flow & 4 &
Released pi-Qwen-Image $4$-step policy model, FlowMap schedule as in the official adapter. \\
MrFlow\,+\,Pi-Flow & $4,1$ &
LR Pi-Flow $4$ steps, then Real-ESRGAN $\times2$ and one HR refinement step. \\
\bottomrule
\end{tabular}
\end{center}
\end{table*}

For the quantitative evaluation, we uniformly set it at $1024\times1024$ resolution, as this is the most common resolution setting in the literature. The evaluation metrics adopt three classes: Geneval, DPG-Bench, and OneIG-Bench, which characterize text-to-image quality from three angles respectively: compositional semantics, dense alignment under long prompts, and multi-dimensional attributes under both Chinese and English bilinguals, including character rendering and style consistency. For some analytical experiments, we additionally generate at officially suggested resolutions of the model such as $1328\times1328$ and $1584\times1056$. The experiments are conducted on Nvidia A100 GPUs, with each experiment run on a single GPU. The reported speedups are all end-to-end actual acceleration, including text encoding, random noise generation, VAE encoding/decoding, SR, and all diffusion forward passes. This means that the speedups of some methods may be lower than those in their original papers, as their reported timings often focus on latent-space diffusion while omitting method-specific transition overheads, such as the region selection, token rearrangement, and stage-wise re-noising in RALU, or the spectral expansion with timestep realignment in SPEED.

\section{Stage-wise Analysis}
\label{app:stage_analysis}

This appendix unfolds in the order of the stages of MrFlow in Section~\ref{sec:method}, supplementing the design choices in the main text with complete experimental evidence and theoretical characterization. Sections~\ref{app:c1} through \ref{app:c4} correspond one-to-one with Sections~\ref{sec:lr_generation} through \ref{sec:hr_refinement}, while Section~\ref{app:c5} separately discusses the pervasive division of labor in which the low-resolution stage determines the structure and the high-resolution stage refines the details. To avoid interrupting the narrative of the main text, the formal propositions and their derivations are also arranged collectively in this appendix.

\subsection{Sources of Acceleration at the Low-Resolution Stage}
\label{app:c1}

This section provides the quantitative basis for the claim in Section~\ref{sec:lr_generation} that ``the low-resolution stage is the main acceleration source of MrFlow''. The benefits brought by low resolution can be decomposed into two mutually independent aspects: one is that each sampling step itself is more time-saving, and the other is that fewer sampling steps are required to reach a structure indistinguishable from high resolution. We examine the two separately on Qwen-Image with the prompt set described in Section~\ref{sec:method}.

\textbf{Single-step inference acceleration.} When each side of the image is shrunk by $2\times$, the number of image tokens is reduced by $4\times$. Considering that the number of text tokens during diffusion inference is far smaller than that of image tokens, and that self-attention depends quadratically on the number of image tokens, the measured comprehensive speedup of single-step inference is about $4\times$; the theoretical upper bound of the self-attention part can reach $16\times$, but the overall speedup is diluted to this magnitude by linear operators such as the feed-forward layers.

\textbf{Few-step convergence.} Compared with the direct saving of single-step cost, why low resolution can lock the structure within fewer timesteps is more subtle. We understand this phenomenon from two complementary angles: the low-resolution stage utilizes the text condition more fully, and the ODE path it needs to traverse is shorter.

First, at low resolution the image tokens depend more strongly on the text condition, and can therefore spread the prompt semantics into the overall structure more quickly. To characterize this, for the attention distribution $\mathbf{A}^{(\ell)}\in[0,1]^{N\times N}$ of the $\ell$-th transformer layer, denoting the image-query set as $\mathcal{Q}_{\mathrm{img}}$ and the text-key set as $\mathcal{K}_{\mathrm{txt}}$, we define the text-to-image interaction strength of that layer as the attention mass that the image queries cast onto the text keys:
\begin{equation}
\mathcal{T}^{(\ell)} \;\triangleq\; \frac{1}{|\mathcal{Q}_{\mathrm{img}}|}\sum_{q\in\mathcal{Q}_{\mathrm{img}}}\sum_{k\in\mathcal{K}_{\mathrm{txt}}} \mathbf{A}^{(\ell)}_{q,k}.
\end{equation}
The larger this quantity, the more the image tokens of that layer tend to obtain information through the text condition rather than their own neighborhood. Table~\ref{tab:c1_attn} gives the measurement results averaged over $3$ prompts and $20$ generation trajectories. At all measured layers, namely $\ell\in\{15,30,45\}$, the $\mathcal{T}^{(\ell)}$ of low resolution is higher than that of high resolution, with the gap being especially significant at the shallow and middle layers, indicating that the low-resolution stage allocates a higher proportion of attention mass to the text condition at each inference step.

\begin{table}[t]
\caption{Layer-wise text-to-image attention mass $\mathcal{T}^{(\ell)}$ of Qwen-Image under $10$-step pure txt2img.}
\label{tab:c1_attn}
\begin{center}
\setlength{\tabcolsep}{9.mm}
\begin{tabular}{ccc}
\toprule
Layer $\ell$ & $\mathcal{T}^{(\ell)}_{R=512}$ & $\mathcal{T}^{(\ell)}_{R=1024}$ \\
\midrule
15 & 0.342 & 0.261 \\
30 & 0.378 & 0.356 \\
45 & 0.198 & 0.174 \\
\bottomrule
\end{tabular}
\end{center}
\end{table}

This attention-level difference is corroborated in the final generation quality. We take the CLIP similarity between the generated image and the prompt:
\begin{equation}
\mathrm{CLIP}(\mathbf{x},\mathbf{c}) \;\triangleq\; \cos\!\big(\boldsymbol{\phi}_I(\mathbf{x}),\,\boldsymbol{\phi}_T(\mathbf{c})\big),
\end{equation}
as a downstream indicator of semantic agreement, where $\boldsymbol{\phi}_I$ and $\boldsymbol{\phi}_T$ are the image and text encoders of CLIP~\citep{CLIP} respectively. Table~\ref{tab:c1_clip} compares, under an equal step-count budget, the baseline $\texttt{HR11}$ that generates directly at high resolution with the two-stage schedule $\texttt{LR10\_HR1}$ dominated by low resolution, where $\mathrm{CLIP}_{\mathrm{ref}}$ is the CLIP image similarity between the generated result and the $28$-step high-resolution reference image. The two-stage schedule attains a higher CLIP score while halving the latency, echoing the stronger text interaction in Table~\ref{tab:c1_attn}, jointly indicating that the low-resolution stage can more quickly depict the overall structure aligned with the text.

\begin{table}[t]
\caption{CLIP evaluation of Qwen-Image under an equal total step-count budget, averaged over $3$ prompts $\times\,3$ seeds.}
\label{tab:c1_clip}
\begin{center}
\setlength{\tabcolsep}{7.mm}
\begin{tabular}{lccc}
\toprule
Schedule & Latency / s & $\mathrm{CLIP}_{\mathrm{ref}}\!\uparrow$ & $\mathrm{CLIP}_{\mathrm{txt}}\!\uparrow$ \\
\midrule
$\texttt{HR11}$      & 10.96 & 0.897 & 0.359 \\
$\texttt{LR10\_HR1}$ & \textbf{5.08} & \textbf{0.935} & \textbf{0.360} \\
\bottomrule
\end{tabular}
\end{center}
\end{table}

Second, the ODE path that low resolution needs to traverse is shorter. For a latent trajectory $\{\mathbf{z}_{t_k}\}_{k=0}^{K}$ produced by sampling, define its path length as:
\begin{equation}
\mathcal{L}\!\big(\{\mathbf{z}_{t_k}\}\big) \;\triangleq\; \sum_{k=0}^{K-1}\big\lVert \mathbf{z}_{t_{k+1}} - \mathbf{z}_{t_k}\big\rVert_2,
\end{equation}
namely the accumulated displacement of the ODE polyline in the latent space. To separate out the part needed to depict the global skeleton, we apply a frequency-domain low-pass operator $\mathcal{P}_{\mathrm{low}}$ once to each sampled latent, obtaining a trajectory $\{\mathcal{P}_{\mathrm{low}}\mathbf{z}_{t_k}\}$ that carries only low-frequency information. Specifically, $\mathcal{P}_{\mathrm{low}}$ is obtained by performing a 2D DFT on each latent channel, retaining the components whose radial frequency does not exceed $\rho_c=0.35\,\rho_{\max}$, and then applying the inverse transform. Table~\ref{tab:c1_pathlen} compares the path length of the original trajectory with that of the low-pass trajectory; at each resolution the low-pass trajectory stays stably at about $58\%$ of the original trajectory, meaning that about $42\%$ of the displacement in a complete ODE trajectory is used to depict high-frequency details.

\begin{table}[t]
\caption{Path length of $30$-step pure txt2img trajectories of Qwen-Image, averaged over $3$ prompts $\times\,4$ seeds.}
\label{tab:c1_pathlen}
\begin{center}
\setlength{\tabcolsep}{5.mm}
\begin{tabular}{lccc}
\toprule
Resolution & $\mathcal{L}\!\big(\{\mathbf{z}_{t_k}\}\big)$ & $\mathcal{L}\!\big(\{\mathcal{P}_{\mathrm{low}}\mathbf{z}_{t_k}\}\big)$ & Ratio \\
\midrule
512  & 270.64 & 154.84 & 57.2\% \\
768  & 404.56 & 232.78 & 57.5\% \\
1024 & 539.52 & 311.37 & 57.7\% \\
\bottomrule
\end{tabular}
\end{center}
\end{table}

The target distribution of low-resolution inference is, in terms of bandwidth, equivalent to applying a low-pass once to the high-resolution distribution, so low-resolution inference only needs to traverse the low-frequency part of the path length to reach the target, which agrees with the decomposition in the table above where the low-frequency displacement accounts for only $58\%$, indicating that low resolution can lock the global skeleton with a shorter ODE path. We therefore use this frequency-domain characterization as an upper-bound estimate of the ODE-trajectory cost; together with the aforementioned observation on text utilization, it explains why the low-resolution stage can converge in fewer steps.

\subsection{Two Design Choices of Pixel-Space Super-Resolution}
\label{app:c2}

This section provides support respectively for the two design choices of the super-resolution stage in Section~\ref{sec:pixel_sr}, namely upsampling in the pixel space rather than the latent space, and choosing a GAN-class rather than an L2-class super-resolution model.

\textbf{The pixel space as the upsampling domain.} Section~\ref{sec:pixel_sr} advocates performing upsampling in the pixel space rather than the latent space. Directly enlarging the feature map in the latent space would destroy the local spatial statistics on which the VAE decoder relies, causing the decoded result to exhibit regular grid-like artifacts, a phenomenon clearly visible in the side-by-side comparison between MrFlow and LSSGen in Appendix~\ref{app:d1}. In contrast, the pixel space is the domain in which various super-resolution models natively work; when the super-resolution output is re-encoded back to the latent space through the VAE, the encoder attenuates the high-frequency components in the pixel domain that exceed the training distribution, so that the resulting latent preserves the low-frequency structure while providing a robust starting point for subsequent low-strength noising and high-frequency resampling. The effectiveness of this design is further supported quantitatively by the frequency-band ablation below and the formal analysis in Section~\ref{app:c3}.

\textbf{Frequency-band alignment of the GAN loss.} Why one should choose a GAN-class super-resolution depends on whether the bias it introduces can be eliminated by the downstream refine stage. Since the refine works on the side close to the clean image, its correction capability for noise of different frequency bands is itself asymmetric, so we design the following ablation to directly measure this capability boundary: when adding forward noise in the latent space, we replace the original Gaussian noise with three classes of structured noise, namely high-frequency band-pass, low-frequency band-pass, and channel-correlated, and then observe whether the output after refine can return to the vicinity of the HR clean cloud. If the noise of a certain band has a basically unchanged distance after refine, it indicates that the bias of that band falls within the correctable range of the refine; conversely, if the distance expands significantly, it indicates that the bias of that band cannot be corrected. We measure the kNN distance of the refined output to the HR cloud in the DINOv2~\citep{DINOv2} feature space, with the results summarized in Table~\ref{tab:c2_noise}.

\begin{table}[t]
\caption{kNN distance of the refine output to the HR cloud in the DINOv2 space under different forward noise components on Qwen-Image, averaged over $3$ prompts $\times\,3$ seeds.}
\label{tab:c2_noise}
\begin{center}
\setlength{\tabcolsep}{7.mm}
\begin{tabular}{lcc}
\toprule
Noise component & $s=0.2$ & $s=0.3$ \\
\midrule
SR input (no refine) & 0.155 & 0.155 \\
Gaussian            & \textbf{0.154} & \textbf{0.155} \\
High-freq           & 0.150 & 0.187 \\
Channel-correlated  & 0.164 & 0.180 \\
Low-freq            & 0.252 & 0.498 \\
\bottomrule
\end{tabular}
\end{center}
\end{table}

The distances of Gaussian and high-frequency noise after refine are almost identical to the default SR starting point, whereas the low-frequency noise rises to $0.498$ at strength $0.3$, expanding by more than $200\%$ relative to the default value. This indicates that the velocity field learned at the refine stage mainly undertakes resampling in the high-frequency directions, and is powerless to correct the bias falling in low frequency. This asymmetry explains the preference of Section~\ref{sec:pixel_sr} for the super-resolution model. Interpolation and L2-trained super-resolution can preserve the coarse layout, but their dominant defects are high-frequency attenuation and diffuse blur rather than locally plausible high-frequency residuals. In contrast, the GAN loss gives up the minimum mean-squared error at the pixel level in exchange for sharp and structurally reasonable details, making the remaining bias more localized to the high-frequency band where the refine is effective. Combined with the finding in Section~\ref{app:c1} that the low-pass trajectory accounts for only $58\%$ of the total path length, the working energy of the refine velocity field indeed concentrates in high frequency, so the advantage of GAN-class super-resolution over L2-class super-resolution and simple interpolation in final quality is consistent with the visual ranking of the three in Section~\ref{sec:ablation}.

\subsection{Formal Characterization of High-Frequency Resampling}
\label{app:c3}

This section provides a formal characterization for the high-frequency resampling induced by the low-strength noising in Section~\ref{sec:noise_resampling}. Recall that $\mathbf{x}_{\mathrm{SR}}$ denotes the pixel-space SR output, $\mathbf{z}_0^{\mathrm{SR}}=\mathcal{E}(\mathbf{x}_{\mathrm{SR}})$ denotes its VAE-reencoded latent, and $\mathbf{z}_t^{\mathrm{HR}}=(1-\sigma_t)\mathbf{z}_0^{\mathrm{SR}}+\sigma_t\boldsymbol{\epsilon}$ is the actual noised high-resolution latent fed to the high-resolution refine stage. In the analysis below, $\mathbf{z}_0$ denotes a clean high-resolution latent drawn from the model's native latent distribution $p_0$, while $\mathbf{Z}_0$ and $\mathbf{Z}_t$ denote the corresponding random variables under the standard flow noising model $\mathbf{Z}_t=a\mathbf{Z}_0+b\boldsymbol{\epsilon}$, where $a=1-\sigma_t$ and $b=\sigma_t$. The core is a lower-bound inequality that guides the choice of the noising level; it is derived from the posterior mean of single-step denoising, i.e., the Tweedie/Wiener estimate, and uses the measurable high-frequency power of the clean latent as the local scale of resampling. We first give the direction-wise decomposition of the posterior mean, derive the lower bound from it, then specialize it to the power spectrum and measure it on the backbone of this paper, and finally take a distribution-level view as a supplementary perspective.

\textbf{Weighted decomposition of the posterior mean.} We first establish the object that single-step denoising computes. For a noisy observation, the clean-latent estimate predicted by the flow matching velocity network equals, in the MMSE sense, the posterior mean of the clean latent conditioned on that observation, an equivalence known as the Tweedie formula \citep{Efron2011, Vincent2011, Karras2022}. Therefore, to characterize how denoising treats the SR output, one only needs to characterize the posterior mean $\mathbb{E}[\mathbf{Z}_0\mid\mathbf{Z}_t=\mathbf{z}_t^{\mathrm{HR}}]$ induced by the actual noised high-resolution latent $\mathbf{z}_t^{\mathrm{HR}}$.

To obtain a closed form of this posterior mean, we need to do one piece of local modeling for the prior $p_0$ of the clean latent. $\mathbf{z}_0^{\mathrm{SR}}$ already carries the correct global structure, and what is truly undetermined is the local bias within its neighborhood, so we approximate $p_0$ in the neighborhood of $\mathbf{z}_0^{\mathrm{SR}}$ as an anisotropic Gaussian $\mathcal{N}(\boldsymbol{\mu},\Lambda)$. Anisotropy is necessary here: the energy of natural images is extremely unevenly distributed across different spatial frequencies, with the variance of the low-frequency structural directions far larger than that of the high-frequency detail directions, and a single isotropic covariance cannot reflect this difference. We further expand in the eigenbasis of $\Lambda$, denoting $\Lambda=\mathrm{diag}(\lambda_i)$, where $\lambda_i$ is the variance of the clean data along the eigendirection $i$. The role of taking the eigenbasis is to diagonalize the covariance so that the directions are statistically uncorrelated, whereby the high-dimensional posterior mean decouples along each direction into a set of independent scalar problems, and the following can solve each direction $i$ separately.

The solution of the direction-wise problem relies on the conditional-expectation formula for jointly Gaussian variables. Let scalars $X$ and $Y$ be jointly Gaussian; then the conditional expectation of $X$ given an observation $Y=y$ is \citep[Eq. (2.81)]{Bishop2006}:
\begin{equation}
\mathbb{E}[X\mid Y=y]=\mathbb{E}[X]+\frac{\mathrm{Cov}(X,Y)}{\mathrm{Var}(Y)}\,\big(y-\mathbb{E}[Y]\big).
\end{equation}
Its structure is to take the prior mean $\mathbb{E}[X]$ as the baseline and make a linear correction to the observation deviation $y-\mathbb{E}[Y]$ by the regression coefficient $\mathrm{Cov}(X,Y)/\mathrm{Var}(Y)$, with the correction magnitude increasing as the correlation between the two strengthens and as the observation's own variance decreases. Proposition~\ref{prop:weighted} is exactly the specialization of this formula to the high-resolution latent obtained by noising the SR latent.

\begin{proposition}[Weighted form of high-resolution refinement]
\label{prop:weighted}
Under the above local Gaussian model, taking the actual noised high-resolution latent $z^{\mathrm{HR}}_{t,i}=a\,z^{\mathrm{SR}}_{0,i}+b\,\epsilon_i$ as the observation, the posterior mean along direction $i$ is:
\begin{equation}
\hat z_{0,i}=\kappa_i\,z^{\mathrm{SR}}_{0,i}+(1-\kappa_i)\,\mu_i+\eta_i\,\epsilon_i,
\qquad
\kappa_i=\frac{a^2\lambda_i}{a^2\lambda_i+b^2}=\frac{\mathrm{SNR}_i}{1+\mathrm{SNR}_i},\quad
\eta_i=\frac{b}{a}\,\kappa_i,
\end{equation}
where $\mathrm{SNR}_i=a^2\lambda_i/b^2$.
\end{proposition}

\begin{proof}
Take $X=Z_{0,i}$ and $Y=Z_{t,i}=aZ_{0,i}+b\epsilon_i$ under the standard noising model. Both are linear combinations of the Gaussian $Z_{0,i}$ and the Gaussian noise $\epsilon_i$, hence jointly Gaussian, and the above conditional-expectation formula applies. Its three statistics are computed one by one as follows. The observation mean is obtained by noting $\mathbb{E}[\epsilon_i]=0$, $\mathbb{E}[Z_{0,i}]=\mu_i$, giving $\mathbb{E}[Z_{t,i}]=a\mu_i$. The observation variance is obtained from the independence of $Z_{0,i}$ and $\epsilon_i$ with respective variances $\lambda_i$ and $1$:
\begin{equation}
\mathrm{Var}(Z_{t,i})=a^2\lambda_i+b^2.
\end{equation}
The noise term in the covariance is eliminated due to independence, leaving only the signal term:
\begin{equation}
\mathrm{Cov}(Z_{0,i},Z_{t,i})=\mathrm{Cov}(Z_{0,i},\,aZ_{0,i})=a\lambda_i.
\end{equation}
Substituting into the conditional-expectation formula gives:
\begin{equation}
\mathbb{E}[Z_{0,i}\mid Z_{t,i}=z^{\mathrm{HR}}_{t,i}]=\mu_i+\frac{a\lambda_i}{a^2\lambda_i+b^2}\,(z^{\mathrm{HR}}_{t,i}-a\mu_i).
\end{equation}
Finally, substitute the actual noised high-resolution observation $z^{\mathrm{HR}}_{t,i}=a\,z^{\mathrm{SR}}_{0,i}+b\,\epsilon_i$. Denoting $\kappa_i=a^2\lambda_i/(a^2\lambda_i+b^2)$, the regression coefficient $a\lambda_i/(a^2\lambda_i+b^2)=\kappa_i/a$, hence
\begin{equation}
\hat z_{0,i}=\mu_i+\frac{\kappa_i}{a}\big(a\,z^{\mathrm{SR}}_{0,i}+b\,\epsilon_i-a\mu_i\big)
=\mu_i+\kappa_i\big(z^{\mathrm{SR}}_{0,i}-\mu_i\big)+\frac{b}{a}\kappa_i\,\epsilon_i.
\end{equation}
Rearranging gives the SR-estimate coefficient $\kappa_i$, the prior-mean coefficient $1-\kappa_i$, and the noise coefficient $\eta_i=\tfrac{b}{a}\kappa_i$.
\end{proof}

$\kappa_i$ is the classical Wiener gain, characterizing the weight that the SR estimate occupies in the reconstruction along direction $i$. After rewriting it as $\kappa_i=\mathrm{SNR}_i/(1+\mathrm{SNR}_i)$, its behavior is clear at a glance: when $\lambda_i$ is large it corresponds to low-frequency structure, the signal-to-noise ratio is high, $\kappa_i\to1$, and the SR value is preserved; when $\lambda_i$ is small it corresponds to high-frequency detail, the signal-to-noise ratio is low, $\kappa_i\to0$, the reconstruction returns to the prior, and the detail in that direction is resampled. The gain $\kappa_i$ is controlled by the clean-data variance and the injected noise level, while the actual discrepancy between $z^{\mathrm{SR}}_{0,i}$ and the corresponding clean latent affects the residual after reconstruction and the validity of the local approximation. Thus, the lower bound below should not be interpreted as correcting arbitrary SR errors; it characterizes the noise scale needed for the prior to take over a high-frequency band once the SR output is already locally plausible.

This also clarifies the role of the SR module in the theory. The bound does not rank SR methods by their average distortion; it identifies the frequency regime in which low-strength noising can be effective. Interpolation and regression-based SR can preserve the coarse low-frequency layout, but their residuals are dominated by missing or averaged high-frequency content and by diffuse blur that often spreads into intermediate frequencies. Such errors are not simply local high-frequency samples to be resampled. In contrast, GAN-based SR keeps the output closer to the natural high-resolution image manifold and supplies sharp high-frequency content, even if some local details are semantically imperfect. Its remaining uncertainty is therefore better matched to the low-variance high-frequency directions where a small $\sigma_t$ can make the high-resolution prior take over. The theoretical lower bound and the empirical SR comparison are connected through this residual-frequency condition.

\textbf{The noise-level lower bound.} We now specialize the above decomposition to the operating regime targeted by MrFlow: the SR latent largely inherits the low-frequency layout from the low-resolution stage, and the residual uncertainty is concentrated in high-frequency directions. Suppose such residuals concentrate in the eigendirection set $\mathcal{H}$, on which the clean-data variance is denoted $\lambda_{\mathrm{hf}}=\max_{i\in\mathcal{H}}\lambda_i$, taking the supremum to cover the worst direction within that band. Proposition~\ref{prop:weighted} shows that whether direction $i$ is dominated by the SR or returns to the prior depends on the relative amplitude of the injected noise and the clean signal in that direction: for the data prior to have the capability to cover and correct the potential errors of the SR on $\mathcal{H}$, a sufficient condition is that the amplitude of the injected noise is not lower than that of the clean signal in that band. From the composition of the noised SR latent in Proposition~\ref{prop:weighted}, the standard deviation of the clean signal along direction $i$ is $a\sqrt{\lambda_i}$, and that of the injected noise is $b$, so this condition is written as $b\ge a\sqrt{\lambda_{\mathrm{hf}}}$. It is equivalent to the signal-to-noise ratio $\mathrm{SNR}_{\mathrm{hf}}=a^2\lambda_{\mathrm{hf}}/b^2\le1$, namely the Wiener gain $\kappa_{\mathrm{hf}}\le\tfrac12$, i.e., the SR has at most as much say as the prior in high frequency.

\begin{proposition}[Resampling noise-level lower bound]
\label{prop:lowerbound}
Under the setting of Proposition~\ref{prop:weighted}, the amplitude of the injected noise in the high-frequency error directions is not lower than that of the clean signal in those directions, i.e., $\mathrm{SNR}_{\mathrm{hf}}\le1$ and $\kappa_{\mathrm{hf}}\le\tfrac12$, if and only if:
\begin{equation}
\frac{\sigma_t}{1-\sigma_t}\ \ge\ \sqrt{\lambda_{\mathrm{hf}}}
\qquad\Longleftrightarrow\qquad
\sigma_t\ \ge\ \sigma_t^\star=\frac{\sqrt{\lambda_{\mathrm{hf}}}}{1+\sqrt{\lambda_{\mathrm{hf}}}}.
\end{equation}
Equivalently, in terms of the rectified-flow equivalent signal-to-noise ratio, $\mathrm{SNR}(t)\le 1/\lambda_{\mathrm{hf}}$.
\end{proposition}

Substituting $a=1-\sigma_t$ and $b=\sigma_t$ into the amplitude condition $b\ge a\sqrt{\lambda_{\mathrm{hf}}}$ gives $\sigma_t\ge(1-\sigma_t)\sqrt{\lambda_{\mathrm{hf}}}$, rearranging both ends yields $\sigma_t/(1-\sigma_t)\ge\sqrt{\lambda_{\mathrm{hf}}}$, and solving for $\sigma_t$ gives $\sigma_t\ge\sqrt{\lambda_{\mathrm{hf}}}/(1+\sqrt{\lambda_{\mathrm{hf}}})=\sigma_t^\star$. $\sigma_t^\star$ depends only on the measurable clean-data variance $\lambda_{\mathrm{hf}}$. This bound is a sufficient but not necessary condition: $\sigma_t\ge\sigma_t^\star$ guarantees that the prior has the capability to cover the high frequency and thus correct the potential errors of the SR; when the SR high frequency is itself reliable, a smaller $\sigma_t$, corresponding to $\kappa_{\mathrm{hf}}$ close to $1$ and preserving more SR content, is also feasible. Therefore $\kappa_{\mathrm{hf}}=\tfrac12$ is not a mandatory value, but the physically non-arbitrary reference point of equal power of noise and signal. The other extreme $\kappa_{\mathrm{hf}}\to0$ requires $\sigma_t\to1$, corresponding to a re-generation that discards all SR information, i.e., the fidelity zero-endpoint of the realism-faithfulness trade-off. Combined with the observation that an excessively large noise level would wash away the correct details of the SR and increase the required number of steps, $\sigma_t$ constitutes a two-sided working interval, with the practical value falling on the order of $\sigma_t^\star$, slightly above it to guarantee the correction capability, and yet close to it to preserve the SR details.

\textbf{Spectral specialization and measurement.} The directions $\mathcal{H}$ and the variances $\lambda_i$ in Proposition~\ref{prop:lowerbound} are eigenquantities of the prior covariance $\Lambda$, which cannot be reliably estimated from a small number of samples in $D\sim10^5$ dimensions. We circumvent this difficulty using the approximate spatial stationarity of the latent. When the statistics of a random field are invariant under spatial translation, its covariance is circulant in structure, and the eigenvectors of a circulant matrix are exactly the Fourier basis. Therefore, under the spatial-stationarity approximation, the eigenvectors of $\Lambda$ degenerate to spatial frequency modes, and the eigenvalues $\lambda_i$ degenerate to the radial power spectrum $\lambda(k)$ indexed by frequency. This degeneration turns the estimation of the high-dimensional covariance spectrum into the estimation of the power spectrum, the latter obtainable stably from the $H_H\times W_H$ spatial samples of a single latent. The direction set $\mathcal{H}$ where the SR error lies corresponds to the high-frequency band, hence $\lambda_{\mathrm{hf}}=\lambda(k_{\mathrm{high}})$, and the lower bound of Proposition~\ref{prop:lowerbound} specializes correspondingly to the per-band $\sigma_t^\star(k)=\sqrt{\lambda(k)}/(1+\sqrt{\lambda(k)})$.

We measure $\lambda(k)$ via band-pass FFT on $18$ clean high-resolution latents of shape $16\times128\times128$, dumped from the main experiment in Section~\ref{sec:main_results}, with the variance of each band passed back to the spatial domain to guarantee the absolute scale, and the Parseval check $\sum_k\lambda(k)=\mathrm{Var}(\mathbf{z}_0)$ has a ratio of $1.000$; the results are shown in the table below.

\begin{table}[t]
\caption{Per-band power spectrum and the corresponding noise-level lower bound on the latent of this paper's backbone.}
\label{tab:c3_spectrum}
\begin{center}
\setlength{\tabcolsep}{5.mm}
\begin{tabular}{lcc}
\toprule
Band & $\sqrt{\lambda(k)}$ & $\sigma_t^\star=\sqrt{\lambda}/(1+\sqrt{\lambda})$ \\
\midrule
Low-freq $k<0.10$        & 0.32 & 0.24 \\
Mid-freq $0.10\le k<0.25$ & 0.18 & 0.16 \\
High-freq $k\ge0.25$     & 0.08 & 0.075 \\
\bottomrule
\end{tabular}
\end{center}
\end{table}

The clean power spectrum concentrates strongly in low frequency, with the low-frequency-band standard deviation about four times that of the high-frequency band, confirming the premise that the high-frequency band $\lambda_{\mathrm{hf}}$ is small. For high-frequency residuals, Proposition~\ref{prop:lowerbound} gives $\sigma_t^\star\approx0.075$, consistent with the empirical value $\sigma_t\in[0.1,0.15]$ taken in the main text slightly above the lower bound. This value is determined by the measured clean high-frequency power and serves as the local resampling scale once the SR residuals are confined to that band. If the residual instead falls into the mid- or low-frequency bands, the measured lower bound increases to $0.16$ or $0.24$, respectively, and a refinement strength of that order starts to approach partial re-generation. This is why bicubic or $L_2$ upsampling, despite preserving the coarse layout, is poorly matched to MrFlow: their dominant error is not a locally plausible high-frequency sample, but attenuation and blur. This also delimits the applicable range of the claim that ``pixel-space rescaling necessarily introduces persistent blur'': it holds for low-frequency or diffuse blur, but not for GAN super-resolution whose residuals are closer to the low-variance high-frequency band.

\textbf{Supplementary distribution-level view.} Propositions~\ref{prop:weighted} and \ref{prop:lowerbound} are single-point, direction-wise characterizations. As a supplementary perspective, noising also draws the SR output distribution closer to the target distribution at a quantifiable rate at the level of the overall distribution. To align this view with the flow noising form in Section~\ref{sec:noise_resampling}, consider the rescaled observation $\tilde{\mathbf{z}}_t^{\mathrm{HR}}=\mathbf{z}_t^{\mathrm{HR}}/(1-\sigma_t)=\mathbf{z}_0^{\mathrm{SR}}+\frac{\sigma_t}{1-\sigma_t}\boldsymbol{\epsilon}$. Let $\hat p_0$ and $p_0$ respectively be the distributions of the SR latent $\mathbf{z}_0^{\mathrm{SR}}$ and the clean latent $\mathbf{z}_0$; the corresponding marginal distributions in this additive-noise view are $\hat p_t=\hat p_0\ast\mathcal{N}(\mathbf{0},(\frac{\sigma_t}{1-\sigma_t})^2\mathbf{I})$ and $p_t=p_0\ast\mathcal{N}(\mathbf{0},(\frac{\sigma_t}{1-\sigma_t})^2\mathbf{I})$.

\begin{proposition}[Distribution-level smoothing after noising]
\label{prop:kl}
Under the above setting:
\begin{equation}
\mathrm{KL}\!\big(\hat p_t\,\big\|\,p_t\big)\;\le\; \frac{(1-\sigma_t)^2}{2\sigma_t^2}\,W_2^2(\hat p_0,\,p_0).
\end{equation}
\end{proposition}

By the Gaussian smoothing inequality \citep[Thm. 1]{PolyanskiyWu2016}, for any $\mu,\nu$ we have $\mathrm{KL}(\mu\ast\gamma_s\,\|\,\nu\ast\gamma_s)\le \tfrac{1}{2s^2}W_2^2(\mu,\nu)$; substituting $\hat p_0$, $p_0$ and $s=\sigma_t/(1-\sigma_t)$ gives the result. This bound, when $\mathrm{KL}=\mathcal{O}(1)$, requires $\sigma_t/(1-\sigma_t)\gtrsim W_2(\hat p_0,p_0)$, a stronger sufficient condition for overall alignment than Proposition~\ref{prop:lowerbound}, corresponding to convolving the two distributions to the scale at which their supports overlap. The direction-wise lower bound of Proposition~\ref{prop:lowerbound}, in contrast, exploits the structural information that the SR global structure is already correct and only the high frequency needs resampling, so the required $\sigma_t$ is far smaller than the overall-alignment scale of Proposition~\ref{prop:kl}, which once again corroborates that low-strength noising is sufficient precisely because the SR only errs in high frequency.

\subsection{Single-Step Sufficiency of High-Resolution Detail Refinement}
\label{app:c4}

This section examines why the high-resolution refine stage can use a single denoising step under the recommended low-strength regime. We first confirm this numerically by comparing it with multi-step references, then explain it from the time-varying characteristics of the velocity field along the denoising trajectory, and finally explain why placing the limited step-count budget at the end of the schedule is the most cost-effective.

\textbf{Numerical single-step sufficiency.} Taking the $8$-step denoising result under the same SR starting point and the same refinement noise seed as the reference, we scan combinations of $K_H\in\{1,2,3,5\}$ and $\sigma_t\in\{0.1,0.3\}$, where $\sigma_t=0.3$ is included as a larger-noise diagnostic setting, and record the CLIP image similarity between the refined image and the reference image:
\begin{equation}
\mathrm{CLIP}_{\mathrm{ref}}(\hat{\mathbf{x}}_0^{(K_H)},\,\hat{\mathbf{x}}_0^{(8)}) \;=\; \cos\!\big(\boldsymbol{\phi}_I(\hat{\mathbf{x}}_0^{(K_H)}),\,\boldsymbol{\phi}_I(\hat{\mathbf{x}}_0^{(8)})\big).
\end{equation}
The results are summarized in Table~\ref{tab:c4_clip}. At $s=0.1$, single-step denoising already attains $0.9974$, only $0.0025$ away from the $0.9999$ of $5$-step denoising, making the additional refinement steps marginal under the recommended setting. By contrast, the larger diagnostic strength $s=0.3$ requires more steps to approach the same reference. This is numerically consistent with the two-sided working interval of strength given by Proposition~\ref{prop:lowerbound}: when the noised SR latent is kept near the clean endpoint, the high-resolution refinement has only weak dependence on $K_H$, and $K_H=1$ is therefore sufficient in our default configuration.

\begin{table}[t]
\caption{CLIP similarity between Qwen-Image high-resolution refine and the $8$-step reference, $3$ prompts.}
\label{tab:c4_clip}
\begin{center}
\setlength{\tabcolsep}{6.mm}
\begin{tabular}{lcccc}
\toprule
Strength $s$ & $K_H=1$ & $K_H=2$ & $K_H=3$ & $K_H=5$ \\
\midrule
0.1 & \textbf{0.9974} & 0.9995 & 0.9997 & 0.9999 \\
0.3 & 0.9902 & 0.9957 & 0.9962 & 0.9996 \\
\bottomrule
\end{tabular}
\end{center}
\end{table}

\textbf{Time-varying characteristics of the velocity field.} The above single-step sufficiency can be explained by the straightness of the velocity field along the denoising trajectory. The local truncation error of one Euler step is approximately:
\begin{equation}
e(t,h) \;\approx\; \tfrac{1}{2}\,h^2\,\big\lVert \dot{\mathbf{v}}_\theta(\mathbf{x}_t,t)\big\rVert,
\end{equation}
where $\dot{\mathbf{v}}_\theta$ is the total differential of the velocity field with respect to the state; the smaller its norm and the straighter the trajectory, the smaller the single-step Euler error. Directly computing $\dot{\mathbf{v}}_\theta$ is expensive, so we instead measure three complementary quantities along the raw-velocity sequence $\{\mathbf{v}_k\}$ of adjacent denoising steps $t_k\to t_{k+1}$. The first is the velocity magnitude:
\begin{equation}
\lVert\mathbf{v}\rVert \;=\; \tfrac{1}{K}\sum_{k}\lVert\mathbf{v}_k\rVert_2;
\end{equation}
the second is the velocity change between adjacent steps, as a difference proxy of $\lVert\dot{\mathbf{v}}_\theta\rVert$:
\begin{equation}
\Delta\lVert\mathbf{v}\rVert \;=\; \tfrac{1}{K-1}\sum_{k}\lVert\mathbf{v}_{k+1}-\mathbf{v}_k\rVert_2;
\end{equation}
the third is the turn angle between adjacent velocity directions, characterizing the degree of trajectory bending:
\begin{equation}
\theta \;=\; \tfrac{1}{K-1}\sum_{k}\arccos\!\frac{\langle\mathbf{v}_{k+1},\,\mathbf{v}_k\rangle}{\lVert\mathbf{v}_{k+1}\rVert\,\lVert\mathbf{v}_k\rVert}.
\end{equation}
We record the above three quantities from the raw velocity predictions of the Qwen-Image denoising transformer during high-resolution refinement, with the results summarized in Table~\ref{tab:c4_velocity}. All three rise monotonically as the strength increases; the smaller the strength, the closer the starting point to $t=0$, the flatter the velocity field, the weaker the trajectory bending, and the smaller the single-step Euler error. This precisely explains the numerical results of Table~\ref{tab:c4_clip}: under small strength, the refinement starts in the flattest segment of the velocity field, so the first Euler step already accounts for the dominant correction.

\begin{table}[t]
\caption{Time-varying statistics of the raw velocity of Qwen-Image high-resolution refine, $3$ prompts $\times\,2$ seeds.}
\label{tab:c4_velocity}
\begin{center}
\setlength{\tabcolsep}{6.mm}
\begin{tabular}{lccc}
\toprule
Strength $s$ & $\lVert \mathbf{v}_\theta\rVert$ & $\Delta\lVert \mathbf{v}_\theta\rVert$ & turn angle / rad \\
\midrule
0.05 & 0.924 & 0.095 & 0.087 \\
0.1  & 0.970 & 0.119 & 0.107 \\
0.2  & 1.002 & 0.140 & 0.126 \\
0.3  & 1.017 & 0.151 & 0.136 \\
0.5  & 1.025 & 0.170 & 0.157 \\
\bottomrule
\end{tabular}
\end{center}
\end{table}

Since the trajectory becomes straighter the closer to $t=0$, the limited step-count budget should be preferentially allocated to this segment. To verify this, we fix the total step count $K_H=6$ for the same SR starting point and compare three step-size distributions: a uniform schedule that spreads the sampling steps evenly, a front-dense schedule that concentrates most steps in the high-noise starting segment, and an end-dense schedule that concentrates most steps in the low-noise end. The trajectory path lengths $\mathcal{L}$ under the three are $215.82$ for uniform, $251.23$ for front-dense, and $\mathbf{157.29}$ for end-dense, with the end-dense saving about $37\%$ of the total displacement relative to the front-dense. This is consistent with the previous two observations of this section: the effective work of the refine concentrates at the low-noise end, and concentrating the step-count budget there likewise completes the detail refinement with the fewest steps. This end-dense schedule is consistent with the timestep shift of SD3, and can be regarded as a special case of the latter extrapolated to the extreme under the condition of SR plus low-strength noising.

\subsection{Structure-Detail Decoupling}
\label{app:c5}

The previous four sections examined the properties of each stage of MrFlow respectively, and this section returns to the design premise pervading the entire pipeline: the low-resolution stage determines the global structure of the image, and the high-resolution stage only refines the details on top of this structure. If this premise holds, then perturbing the respective randomness of the two stages should leave drastically different imprints on different frequency bands, and we directly test this inference with a set of controlled-variable seed-sensitivity experiments.

We fix the prompt set and introduce randomness into the two stages of MrFlow respectively. $\mathrm{vary\_LR}$ fixes the random seed of the HR refinement and only changes the seed of the low-resolution stage; $\mathrm{vary\_HR}$ fixes the seed of the low-resolution stage, i.e., fixes the entire LR and SR procedure before the SR, and only changes the seed of the HR refinement. Each class is run with $4$ seeds to obtain $4$ outputs, and the pairwise MSE is computed pair by pair over all $\binom{4}{2}=6$ pairs of images within the group on $3$ prompts. To separate the differences in structure and detail, we first decompose each image into low-frequency and high-frequency components via a Gaussian low-pass, then repeat the computation on the two components respectively, supplemented by the full-image MSE and the high-frequency/low-frequency MSE ratio, the former reflecting the comprehensive difference and the latter reflecting in which band the difference mainly falls. The results are summarized in Table~\ref{tab:c5}.

\begin{table}[t]
\caption{Seed-sensitivity experiment of the two stages of MrFlow on Qwen-Image, $3$ prompts $\times\,18$ pairs, frequency-domain pairwise MSE.}
\label{tab:c5}
\begin{center}
\setlength{\tabcolsep}{2.mm}
\begin{tabular}{lcccc}
\toprule
Control variable & low-freq MSE & high-freq MSE & full-image MSE & high/low ratio \\
\midrule
$\mathrm{vary\_LR}$ & $1.087\times10^{-2}$ & $3.054\times10^{-2}$ & $4.842\times10^{-2}$ & 7.4 \\
$\mathrm{vary\_HR}$ & $1.298\times10^{-6}$ & $2.584\times10^{-5}$ & $2.575\times10^{-5}$ & 19.8 \\
\bottomrule
\end{tabular}
\end{center}
\end{table}

The difference between the two control variables is extremely stark. The difference introduced by changing the low-resolution seed at the low-frequency scale is about $8000$ times that of changing the high-resolution seed, and about $1900$ times at the full-image scale. This stark asymmetry indicates that the structural content of MrFlow's final output is almost entirely determined by the low-resolution stage, and the high-resolution stage only introduces a small local perturbation on the fixed LR starting point. Moreover, the distributions of the two perturbations across frequency bands also differ: the high-frequency/low-frequency ratio of $\mathrm{vary\_HR}$ is $19.8$, markedly higher than the $7.4$ of $\mathrm{vary\_LR}$, indicating that once the structure is fixed by the low-resolution stage, the residual change introduced by the high-resolution refine concentrates in high frequency, exactly falling within the boundary that detail refinement should have. This set of results provides direct evidence for the pipeline division of labor described in Section~\ref{sec:method}, and also echoes the conclusion in Section~\ref{app:c4} that the refine velocity field only takes effect in the small-strength interval at the low-noise end.

\section{MrFlow on Recent Open Models}
\label{app:recent_models}

\begin{table}[t]
\caption{Extended results on FLUX.2 Klein variants at $1024\times1024$ resolution. Speedups of non-base variants are measured relative to the corresponding Klein Base $50$-step model with the same model scale. Here ``$\times2$'' denotes the effective positive/negative CFG branches.}
\label{tab:flux2_klein}
\begin{center}
\setlength{\tabcolsep}{0.3mm}
\begin{tabular}{lccccccc}
\toprule
Model variant & Method & Training-free & NFEs & Speedup $\uparrow$ & Geneval $\uparrow$ & DPG $\uparrow$ & OneIG-En $\uparrow$ \\
\midrule
FLUX.2-klein-base-4B & Native & - & $50{\times}2$ & 1x & 0.78 & 83.05 & 0.49 \\
FLUX.2-klein-base-4B & MrFlow & $\surd$ & $(20, 1){\times}2$ & 5.16x & 0.75 & 83.24 & 0.49 \\
FLUX.2-klein-base-4B & MrFlow & $\surd$ & $(12, 1){\times}2$ & 8.03x & 0.74 & 82.63 & 0.48 \\
\cdashline{1-8}
FLUX.2-klein-4B & Native & - & 4 & 20.5x & 0.84 & 85.78 & 0.49 \\
FLUX.2-klein-4B & MrFlow$^\ddagger$ & $\ddagger$ & 4, 1 & 20.1x & 0.83 & 85.17 & 0.50 \\
\midrule
FLUX.2-klein-base-9B & Native & - & $50{\times}2$ & 1x & 0.83 & 85.84 & 0.54 \\
FLUX.2-klein-base-9B & MrFlow & $\surd$ & $(20, 1){\times}2$ & 5.40x & 0.79 & 85.09 & 0.54 \\
FLUX.2-klein-base-9B & MrFlow & $\surd$ & $(12, 1){\times}2$ & 8.79x & 0.77 & 84.62 & 0.54 \\
\cdashline{1-8}
FLUX.2-klein-9B & Native & - & 4 & 22.5x & 0.84 & 86.31 & 0.54 \\
FLUX.2-klein-9B & MrFlow$^\ddagger$ & $\ddagger$ & 4, 1 & 26.9x & 0.87 & 85.92 & 0.54 \\
\bottomrule
\end{tabular}
\end{center}
\end{table}

\begin{table}[t]
\caption{Extended results on Z-Image and Z-Image-Turbo at $1024\times1024$ resolution. Speedups of Z-Image-Turbo variants are measured relative to Z-Image Base at the standard $50$-step setting. NFEs follow the effective DiT-forward count; for Z-Image-Turbo, the official $9$-step setting performs $8$ DiT forwards. Here ``$\times2$'' denotes the effective positive/negative CFG branches.}
\label{tab:zimage_recent}
\begin{center}
\setlength{\tabcolsep}{0.15mm}
\begin{tabular}{lcccccccc}
\toprule
Model variant & Method & Training-free & NFEs & Speedup$\uparrow$ & Geneval$\uparrow$ & DPG$\uparrow$ & OneIG-En$\uparrow$ & OneIG-Zh$\uparrow$ \\
\midrule
Z-Image & Native & - & $50{\times}2$ & 1x & 0.75 & 86.90 & 0.56 & 0.52 \\
Z-Image & MrFlow & $\surd$ & $(20, 1){\times}2$ & 6.66x & 0.68 & 85.67 & 0.55 & 0.51 \\
Z-Image & MrFlow & $\surd$ & $(12, 1){\times}2$ & 10.8x & 0.66 & 84.97 & 0.51 & 0.46 \\
\cdashline{1-9}
Z-Image-Turbo & Native & - & 8 & 10.3x & 0.76 & 85.05 & 0.52 & 0.49 \\
Z-Image-Turbo & MrFlow$^\ddagger$ & $\ddagger$ & 8, 1 & 21.0x & 0.75 & 84.37 & 0.52 & 0.49 \\
\bottomrule
\end{tabular}
\end{center}
\end{table}

To further examine whether the same staged design transfers beyond the two main backbones, we additionally evaluate MrFlow on two recent open model families, FLUX.2 Klein~\citep{flux-2-2025} and Z-Image~\citep{ZImage}. The results are summarized in Tables~\ref{tab:flux2_klein} and~\ref{tab:zimage_recent}. For the non-base FLUX.2 Klein variants and Z-Image-Turbo, which already use a reduced-step generation regime, we report the speedup relative to the corresponding base model at the standard $50$-step setting, so that the number reflects the practical acceleration limit against the full model in the same family.

On FLUX.2 Klein Base, MrFlow keeps the same behavior observed on FLUX.1-dev and Qwen-Image. With $20$ low-resolution steps and $1$ high-resolution step, it reaches $5.16\times$ and $5.40\times$ speedup on the 4B and 9B variants, respectively, while preserving the OneIG-En score within a very small gap. The more aggressive $12,1$ configuration further improves the speedup to $8.03\times$ and $8.79\times$, with a moderate metric drop. The distilled Klein variants also remain compatible with the multi-resolution pipeline: adding one high-resolution refinement step to the $4$-step model reaches $20.10\times$ on 4B and $26.92\times$ on 9B relative to their base counterparts, denoted as MrFlow$^\ddagger$ in the table.

The conclusion is similar on Z-Image. On the base model, the $20,1$ setting achieves $6.66\times$ speedup while keeping OneIG-En at $0.55$ and OneIG-Zh at $0.51$, close to the native model. Even the $12,1$ setting reaches $10.8\times$ speedup while maintaining usable metric performance. When combined with Z-Image-Turbo, MrFlow further pushes the effective speedup to $21.0\times$ under the $8,1$ effective DiT-forward setting. These results indicate that MrFlow is not tied to a particular architecture or model release, and can be directly applied to newer flow-matching backbones and their distilled variants.

\section{Efficiency Analysis}
\label{app:efficiency}

We decompose the end-to-end runtime of MrFlow to clarify where the acceleration comes from. Taking the $12{+}1$ setting on Qwen-Image as an example, we measure all stages in the inference path, including text encoding, low-resolution initial noise generation, low-resolution denoising, VAE decoding, pixel-space super-resolution, VAE encoding, high-resolution noise generation and noising, high-resolution denoising, and the final VAE decoding. The average end-to-end latency of this configuration is $4.77$s, compared with $49.32$s for native $50$-step Qwen-Image, corresponding to a practical speedup of $10.35\times$.

Figure~\ref{fig:efficiency_breakdown} visualizes the same measurement as a left-aligned horizontal runtime breakdown. The dominant costs of MrFlow are low-resolution sampling and high-resolution refinement, which take $3.24$s and $1.03$s, respectively. In contrast, the fixed overhead beyond denoising is small: the low-resolution initial noise generation, intermediate VAE decoding, Real-ESRGAN super-resolution, and VAE encoding together take about $0.30$s. Therefore, the efficiency gain mainly follows from two factors. First, each low-resolution denoising step processes only about $1/4$ of the image tokens. Second, the low-resolution stage can establish the global structure with substantially fewer steps. Since pixel-space super-resolution and the additional VAE transformations introduce only a small fixed cost, MrFlow converts a large portion of expensive high-resolution sampling into cheaper low-resolution sampling while retaining one high-resolution refinement step for local details.

\begin{figure}[t]
\begin{center}
\includegraphics[width=\linewidth]{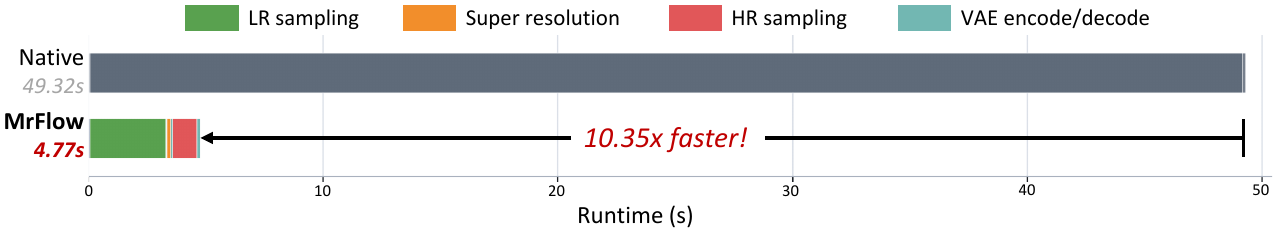}
\end{center}
\caption{Stage-wise runtime breakdown of native Qwen-Image and MrFlow-$12, 1$.}
\label{fig:efficiency_breakdown}
\end{figure}

\begin{table}[t]
\caption{Stage-wise runtime breakdown on Qwen-Image at $1024\times1024$ resolution. Times are reported in seconds. Blank entries indicate stages that are not present in the corresponding pipeline.}
\label{tab:efficiency_breakdown}
\begin{center}
\setlength{\tabcolsep}{0.1mm}
\resizebox{\linewidth}{!}{%
\begin{tabular}{lccccccccccc}
\toprule
Method & NFE & Text enc. & LR noise & LR samp. & Mid VAE dec. & SR & VAE enc. & HR noise & HR samp. & Final VAE dec. & Total \\
\midrule
Native & $50{\times}2$ & $0.067$ & -- & -- & -- & -- & -- & $0.0025$ & $49$ & $0.14$ & $49$ \\
MrFlow & $(12, 1){\times}2$ & $0.067$ & $0.00012$ & $3.2$ & $0.035$ & $0.18$ & $0.083$ & $0.0025$ & $1.0$ & $0.14$ & $4.8$ \\
\bottomrule
\end{tabular}
}
\end{center}
\end{table}

\section{More Generation Examples}
\label{app:examples}

To keep the submitted PDF compact and convenient to download and inspect, the rasterized figures in this paper are compressed with a compression ratio above $90\%$. Some fine-grained visual details may therefore appear less sharp in the PDF, especially in densely arranged comparison figures. This compression artifact is only introduced during document preparation and should not be interpreted as a limitation of the generated images or of the proposed method.

\subsection{Comparison with Various SOTA Strategies}
\label{app:d1}

Beyond the quantitative metrics in Section~\ref{sec:main_results} of the main text, we observe that the gap between methods is even more noticeable when directly inspecting generation examples.

\textbf{Benchmark samples.} In addition to the advantage in the quantitative evaluation of Section~\ref{sec:main_results}, MrFlow in fact also looks better on the actually generated images. Figure~\ref{fig:showcase-compare-flux} directly extracts generation examples from the DPG-Bench task on FLUX.1-dev. MrFlow significantly surpasses other existing training-free strategies in both generation quality and efficiency, and MrFlow$^\dagger$, which fuses timestep distillation without extra training, further attains a higher speedup.

\begin{figure}[t]
\begin{center}
\includegraphics[width=\linewidth]{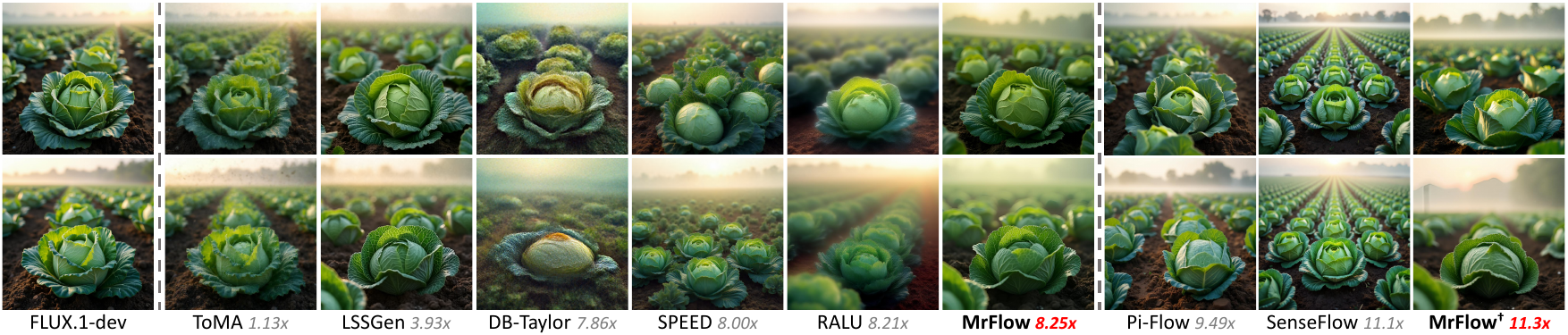}
\end{center}
\caption{Comparison of the effects of various methods on FLUX.1-dev. The dashed lines separate the pretrained model, training-free strategies, and strategies that rely on or exploit timestep distillation.}
\label{fig:showcase-compare-flux}
\end{figure}

\textbf{Multi-resolution strategy comparison.} Figure~\ref{fig:lssgen} further compares MrFlow with other multi-resolution acceleration strategies, including LSSGen, RALU, and SPEED. LSSGen already exhibits dense artifacts even at a moderate speedup, while RALU and SPEED can still obtain competitive quantitative scores on FLUX.1-dev at speedups above $8\times$, with SPEED even partially exceeding MrFlow on some metrics in Table~\ref{tab:flux_tf}. However, the generated samples reveal a clear gap in perceptual fidelity: these baselines exhibit severe local degradation, including blurred fine details, unstable textures, or collapsed structures. In contrast, MrFlow maintains sharp and coherent details at comparable or higher acceleration levels, suggesting that pixel-space super-resolution followed by low-strength high-resolution resampling is more reliable than latent- or frequency-domain expansion.

\begin{figure}[t]
\begin{center}
\includegraphics[width=\linewidth]{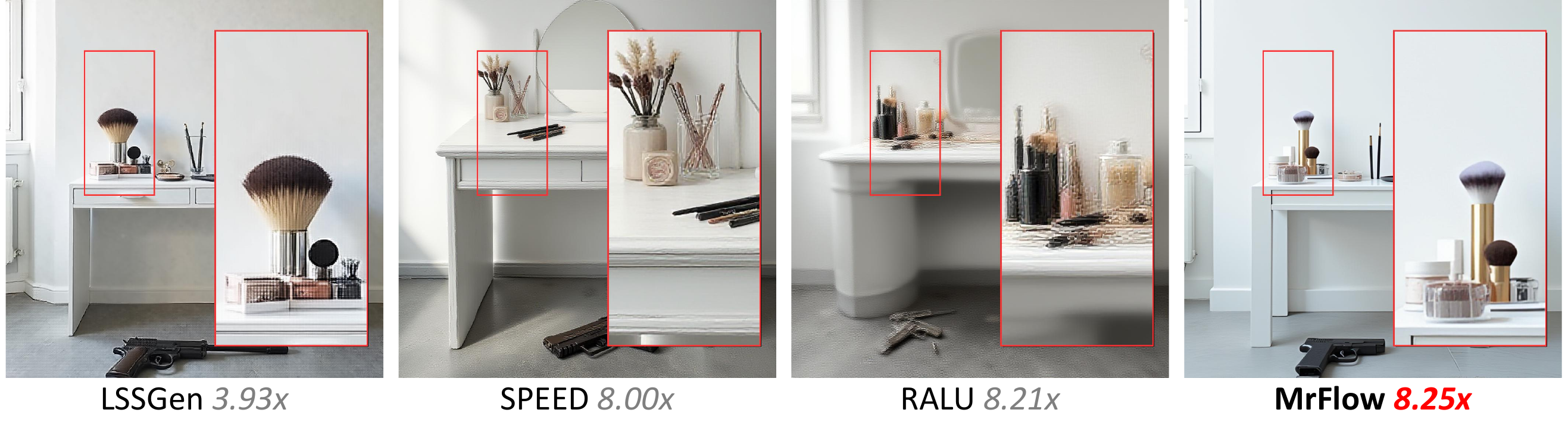}
\end{center}
\caption{Detail comparison of multi-resolution acceleration methods on FLUX.1-dev.}
\label{fig:lssgen}
\end{figure}

\textbf{More scenarios.} We also test some more diverse prompts on Qwen-Image to further examine the accelerated generation capability of each method. Figure~\ref{fig:showcase-compare} shows some examples. Consistent with the observation on FLUX.1-dev, MrFlow can not only significantly outperform various other advanced training-free strategies, but also be directly combined with training-dependent timestep distillation strategies to attain a higher speedup.

\subsection{More Examples of MrFlow}
\label{app:d2}

In Figures~\ref{fig:showcase-mrflow-qwen} and \ref{fig:showcase-mrflow-flux}, we show the generation effects of MrFlow performing acceleration under the configuration of $12$ low-resolution steps and $1$ high-resolution step on FLUX.1-dev and Qwen-Image respectively. At a speedup of $8$-$10\times$, MrFlow demonstrates excellent generation effects at arbitrary aspect ratios and resolutions.

\begin{figure}[t]
\begin{center}
\includegraphics[width=\linewidth]{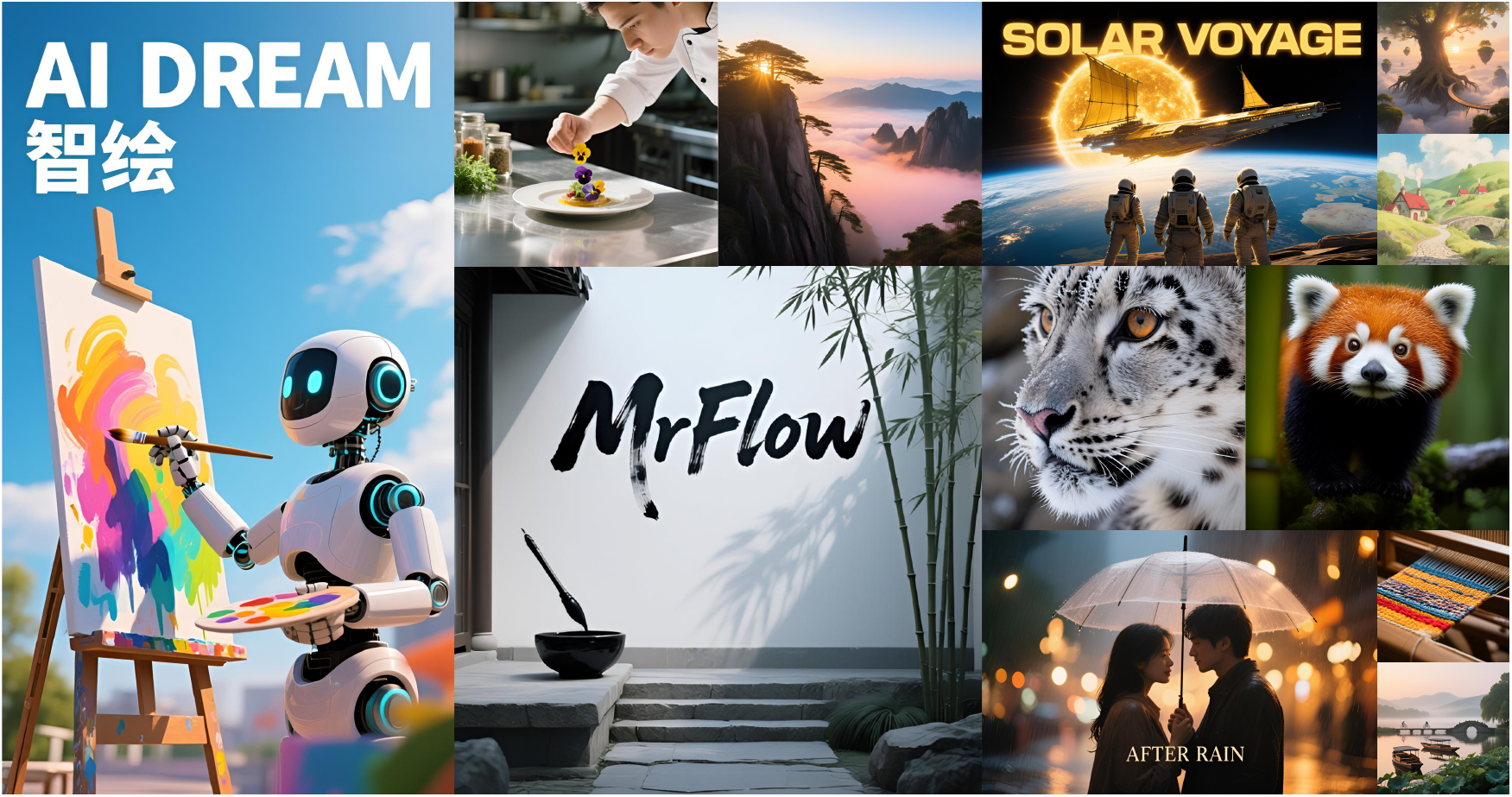}
\end{center}
\caption{Generation examples of MrFlow on Qwen-Image. The configuration is $12$ steps at the low-resolution stage and $1$ step at the high-resolution stage, with speedups all above $10\times$.}
\label{fig:showcase-mrflow-qwen}
\end{figure}

\begin{figure}[t]
\begin{center}
\includegraphics[width=\linewidth]{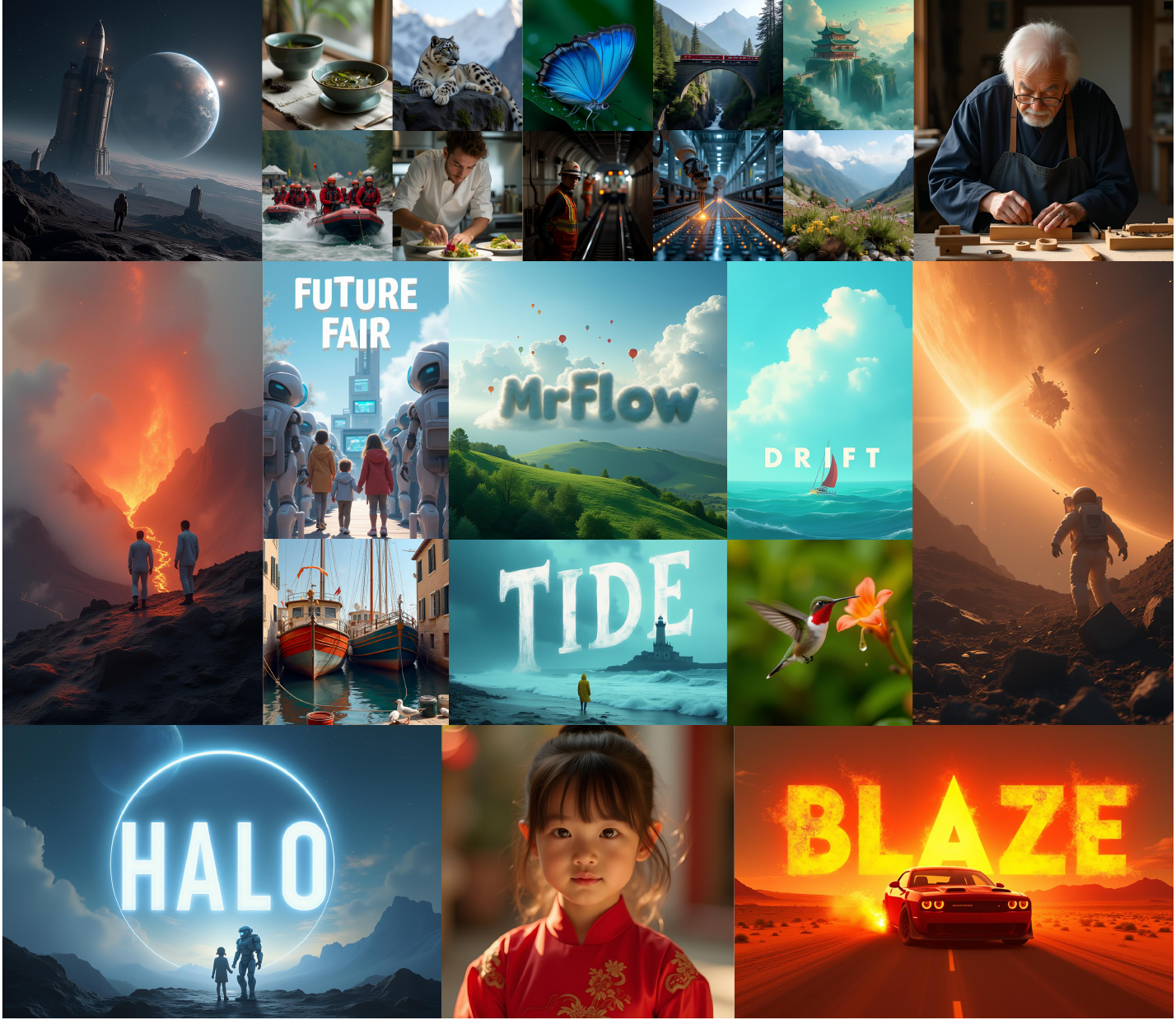}
\end{center}
\caption{Generation examples of MrFlow on FLUX. The configuration is $12$ steps at the low-resolution stage and $1$ step at the high-resolution stage, with speedups all above $8\times$.}
\label{fig:showcase-mrflow-flux}
\end{figure}

\subsection{Prompts of the Images in the Paper}
\label{app:d3}

For reproducibility and easier inspection of qualitative examples, we list below the text prompts used for the images shown in the paper. For figures containing multiple samples, prompts are listed in the same order as the corresponding images.

\paragraph{Prompts for Figure~\ref{fig:showcase-compare}.}
\begin{enumerate}
\item Close-up of a white horse in a spring pasture, sunlight on mane, clear eye reflection, wildflowers, airy realistic photography.
\item Mars rice greenhouse, shallow water beds, red dunes outside, agronomists, robots, warm lamps, detailed sci-fi realism.
\end{enumerate}

\paragraph{Prompt for Figure~\ref{fig:framework}.}
\begin{enumerate}
\item Close-up of a chubby golden British shorthair on a knitted blanket, tiny ears, plush cheeks, bright cozy bedroom light.
\end{enumerate}

\paragraph{Prompt for Figure~\ref{fig:sr_compare}.}
\begin{enumerate}
\item A cozy corner bookstore at dusk with a large wooden hanging sign that reads 'THE WANDERING PAGE' in hand-painted serif letters, and below it a smaller chalkboard sign saying 'Est. 1987 — Open 9am to 9pm'. Warm amber window light, rain-slick cobblestone street, autumn leaves, shallow depth of field, 35mm film photograph.
\end{enumerate}

\paragraph{Prompt for Figure~\ref{fig:mrflow-compare-qwen}.}
\begin{enumerate}
\item Top-down view of a watchmaker's workbench, featuring a disassembled mechanical wristwatch with tiny gears, jewels, screws and springs laid out on a green felt mat. A loupe, tweezers, and a brass screwdriver set sit beside a leather-bound notebook with the handwritten heading 'Caliber 2824-2 — Service Log'. Warm tungsten desk lamp, hard shadows, macro photograph at f/5.6, razor-sharp focus across the scene.
\end{enumerate}

\paragraph{Prompt for Figure~\ref{fig:showcase-compare-flux}.}
\begin{enumerate}
\item An expansive field, blanketed by the soft light of morning, cradles a collection of eight cabbages, their green heads round and plump. These vegetables are nestled among rows of rich soil, dotted with glistening droplets of dew that cling to their crinkled leaves. As wisps of mist begin to lift, the cabbages lie poised, ready for the day's impending harvest.
\end{enumerate}

\paragraph{Prompt for Figure~\ref{fig:lssgen}.}
\begin{enumerate}
\item An immaculate white vanity desk featuring an array of beauty products, among which a soft-bristled cosmetics brush and a sleek black eyeliner pencil are neatly arranged. To the side, a stark and dissonant element, a metallic handgun, lies on the textured surface of a grey concrete floor, creating a jarring juxtaposition against the delicate makeup tools. The vanity itself is positioned against a white wall, illuminated by the natural light streaming in from a nearby window.
\end{enumerate}

\paragraph{Prompts for Figure~\ref{fig:showcase-mrflow-qwen}.}
The items marked with $\dagger$ were originally written in Chinese; they are shown here as English translations or English descriptions to keep the source compatible with pdfLaTeX.
\begin{enumerate}
\item ($\dagger$) Vertical AI art festival poster, a friendly AI robot is painting a glowing mural by itself in daylight, robot arm holding a brush, vivid colors, clean promotional cover composition, no human painter; only readable text: ``AI DREAM'' and a Chinese phrase meaning ``AI painting''; strictly no other letters, no extra words, no logos, no signatures, no captions, no watermark.
\item A focused young chef with rolled-up sleeves carefully placing edible flowers on a finely plated dish in a sunlit open kitchen, gleaming stainless-steel surfaces, herbs and spice jars softly out of focus in the background, 35mm reportage style.
\item ($\dagger$) The main peak of Huangshan at sunrise, bathed in golden morning glow, with ancient pine trees leaning out from steep granite cliffs, a rolling sea of clouds tinted pale orange and rose, distant mountain layers fading into cool blue, ultra-wide panoramic view, sharp details under natural light.
\item ($\dagger$) A courtyard calligraphy wall with the only readable text ``MrFlow'', surrounded by brushes, an ink bowl, stone steps, and bamboo shadows, with no other readable text.
\item Bright sci-fi film poster, three astronauts facing a golden solar sail ship, Earth glow, heroic composition, only readable text: ``SOLAR VOYAGE''.
\item A colossal ancient tree with sprawling roots breaking through layers of clouds at sunrise, tiny floating islands suspended around its branches, sunlight pouring through the foliage onto a winding skybridge, painterly fantasy concept art with a warm saturated palette.
\item A cozy hillside village in Studio Ghibli illustration style, hand-painted clouds drifting over rolling green hills, small red-roofed cottages with smoking chimneys, a winding cobblestone path leading to a stone bridge, warm afternoon light, soft watercolor textures.
\item Extreme close-up of a snow leopard face, frost on whiskers, amber eye, wet nose texture, snowy rocks blurred behind, realistic wildlife photo.
\item Close-up of a red panda on a mossy branch, damp fur, tiny claws, round eyes, bamboo forest rain blur, detailed wildlife realism.
\item Romantic drama poster, two figures under a clear umbrella after rain, warm city lights, soft cinematic bokeh, only readable text: ``AFTER RAIN''.
\item Woven resolution tapestry, coarse threads becoming fine image detail, wooden loom, colored yarn, tactile macro scene.
\item Hangzhou lake at sunrise, wooden boats, lotus leaves, stone bridge, misty hills, cyclists on shore, soft realistic light.
\end{enumerate}

\paragraph{Prompts for Figure~\ref{fig:showcase-mrflow-flux}.}
\begin{enumerate}
\item Orbital lighthouse array, huge mirrors, maintenance craft, Earth glow, astronauts, hard sci-fi panorama.
\item Ceramic table arrangement, celadon bowls, tea, wet leaves, linen, bronze spoon, window light, macro material detail.
\item A majestic snow leopard resting on a moss-covered rock at the edge of a Himalayan cliff, soft alpine sunlight on its thick spotted fur, distant snow-capped peaks in the background, sharp golden eyes, wildlife photograph at 200mm with anamorphic depth.
\item An extreme macro photograph of a morpho butterfly's iridescent blue wing with visible scale tile structure, tiny dewdrops clinging to a few scales acting as natural lenses, soft greenery as a cool blurred background, 1:2 magnification.
\item Red alpine train crossing a stone viaduct above pine forest, waterfalls, misty peaks, tiny hikers, crisp morning light.
\item Jade palace floating above a cloud sea, cranes, waterfalls, tiled roofs, glowing bridges, bright Chinese fantasy painting.
\item River rescue training exercise, inflatable boats, ropes, helmets, fast water, command tents, realistic documentary action.
\item A focused young chef with rolled-up sleeves carefully placing edible flowers on a finely plated dish in a sunlit open kitchen, gleaming stainless-steel surfaces, herbs and spice jars softly out of focus in the background, 35mm reportage style.
\item Subway signal repair, workers in tunnel, orange lights, cables, tools, train blur, realistic infrastructure action.
\item AI chip foundry floor, wafer carriers, robotic arms, cleanroom suits, amber reflections, precise industrial photograph.
\item Alpine seed collection, botanists, tiny flowers, rock slopes, clouds, specimen boxes, crisp realistic macro landscape.
\item An elderly Japanese craftsman with white hair and round wire-rim glasses focused on carving a small wooden box in his sunlit workshop, soft sawdust floating in the warm light, neatly arranged hand tools on the workbench, traditional indigo apron with subtle wear, 50mm photograph with shallow depth of field.
\item Text-free science thriller poster, research lab on the rim of an active volcano, orange lava glow, white steam, scientists in suits, no text, no logos.
\item Promotional event cover, bright outdoor technology fair, robots, image screens, families, clean daylight, only readable text: ``FUTURE FAIR''.
\item Only the word 'MrFlow' sculpted from clouds above green hills, balloons, sunlight, no other readable text.
\item Adventure poster, huge centered title ``DRIFT'' in windblown white letters, small sailboat under a giant cloud, turquoise sea, sunny cinematic poster, no other text.
\item Mediterranean boat repair dock, painted hulls, ropes, gulls, buckets, sunlit water reflections, busy realistic scene.
\item Coastal drama poster, huge top title ``TIDE'' in white painted letters, old lighthouse, crashing waves, yellow raincoat figure, cinematic clouds, no other text.
\item A ruby-throated hummingbird hovering with rapid wing motion blur in front of a bright trumpet vine flower, droplets of nectar caught mid-air, lush greenery softly out of focus behind, 1/4000s shutter, vivid macro photograph.
\item Text-free space blockbuster poster, astronaut reaching toward a damaged satellite above Earth, golden sunlight, debris field, heroic composition, no text, no logos.
\item Hopeful sci-fi cover, huge centered title ``HALO'' in soft luminous letters, child and service robot under a ring-shaped sky station, clean daylight, no other text.
\item Close-up portrait of a little Chinese girl in red hanfu, soft courtyard sunlight, delicate embroidery, warm cinematic bokeh.
\item Action poster, huge centered title ``BLAZE'' in fiery orange letters, red sports car and sparks on a desert road, dramatic sunset, no other text.
\end{enumerate}

\end{document}